\documentclass{article}

\usepackage[eandd, final]{neurips_2026}

\usepackage{microtype}
\usepackage{graphicx}
\usepackage{subcaption}
\usepackage{booktabs} 
\usepackage{makecell}
\usepackage{placeins}

\usepackage{xcolor}
\usepackage[most]{tcolorbox}
\usepackage{float}

\definecolor{lightgreen}{HTML}{DFFFD6}
\definecolor{lightred}{HTML}{FFD6D6}
\definecolor{lightyellow}{HTML}{FFFFD6}
\definecolor{lightblue}{HTML}{D6F0FF}
\definecolor{lightgray}{HTML}{E0E0E0}

\usepackage{hyperref}

\usepackage{amsmath}
\usepackage{amssymb}
\usepackage{mathtools}
\usepackage{amsthm}

\usepackage[capitalize,noabbrev]{cleveref}

\theoremstyle{plain}
\newtheorem{theorem}{Theorem}[section]
\newtheorem{proposition}[theorem]{Proposition}

\theoremstyle{definition}

\theoremstyle{remark}

\usepackage[textsize=tiny]{todonotes}

\usepackage{enumitem}
\usepackage{placeins}
\usepackage{longtable}
\usepackage{caption}

\newcommand{\mycaption}[2]{\caption{\textbf{#1}. #2}}
\newcommand{\mypar}[1]{\medskip\noindent\textbf{#1}}
\newcommand{\benchname}[0]{\textsc{TasteBench} }

\usetikzlibrary{decorations.pathreplacing,calc}
\colorlet{boxaccent}{black}
\newcommand{\key}[1]{\textcolor{boxaccent}{\textbf{#1}}}

\definecolor{D2Teal}{HTML}{1B9E77}
\definecolor{D2Orange}{HTML}{D95F02}
\definecolor{D2Purple}{HTML}{7570B3}
\definecolor{D2Grey}{HTML}{555555} 
\definecolor{D2Green}{HTML}{66a61e} 
\definecolor{D2Blue}{HTML}{7570b3} 
\colorlet{boxaccent}{black}

\newtcolorbox{custombox}[2]{
  blanker,
  left=0.9cm,
  top=5pt,
  bottom=5pt,
  before skip=15pt,
  after skip=15pt,
  coltext=black,
  before upper={
    \colorlet{boxaccent}{#2}
    \setlist[itemize]{
        label=\textcolor{#2}{\textbullet}, 
        leftmargin=12pt, 
        itemsep=8pt
    }
  },
  overlay={
    \node[anchor=center, rotate=90, text=#2, font=\ttfamily] 
      at ($(frame.west)+(10pt,0)$) {\MakeUppercase{#1}};
    \fill[#2, rounded corners=2.5pt] 
      ($(frame.north west)+(23pt,0)$) rectangle ($(frame.south west)+(25pt,0)$);
  }
}

\setlist[enumerate]{leftmargin=20pt}
\setlist[itemize]{leftmargin=15pt} 
\title{\textsc{TasteBench}: Multimodal Benchmark for Sensory Prediction, from Molecules to Sustainable Foods}

\author{%
  Anna T. Thomas\textsuperscript{1,2}\thanks{These authors contributed equally.} \quad
  Sohum Patnaik\textsuperscript{2}\footnotemark[1] \quad
  Caroline Cotto\textsuperscript{3} \quad
  Benjamin Sanchez-Lengeling\textsuperscript{4} \\[0.5em]
  \textsuperscript{1}Stanford Computer Science \\
  \textsuperscript{2}Food Intelligence Lab \quad
  \textsuperscript{3}NECTAR \\
  \textsuperscript{4}University of Toronto Chemical Engineering
}

\begin{document}
\raggedbottom

\maketitle

\begin{abstract}
AI has made rapid progress in modeling human sight and hearing, but taste lags behind. Sensory prediction could accelerate applications such as sustainable protein discovery, where evaluating whether a novel food tastes like its animal-based target requires expensive human sensory panels, bottlenecking the design-build-test loop.
We introduce \textsc{TasteBench}, a multimodal benchmark and privacy-preserving competition for
sensory prediction, spanning two tasks: a food-level ranking task built on 21K+
human evaluations across 215 plant-based foods in 24 product categories, yielding 935 within-category ranking pairs, and a
supporting molecular-level taste classification task over 15K flavor molecules. 
To enable rigorous interpretation of model performance, we characterize the
ground truth: inter-rater agreement among panelists is low (Krippendorff's
$\alpha = .077$), and the split-half reliability ceiling of panel-aggregated rankings
is .825, establishing the range within which ML systems on this benchmark
should be assessed. We evaluate baselines across four input modalities; on
the same pairs panelists rated, the best model achieves .661 pairwise accuracy,
competitive with the median individual panelist (.650), and .683 across all within-category pairs.
\textsc{TasteBench} provides the evaluation infrastructure and baselines for measuring progress on
sensory prediction for sustainable protein discovery. 
Our code is available at \texttt{https://github.com/Food-Intelligence-Lab/tastebench/}. 


\end{abstract}

\section{Introduction}

AI systems have made substantial progress in interpreting images and audio, but predicting human taste perception remains comparatively underexplored.\footnote{Strictly speaking, \emph{taste} refers to gustation (sweet, bitter, sour, salty, umami). For the food-level task, we use it colloquially for the overall sensory experience of food, including aroma, texture, and appearance; for the molecular task, we use it in the strict gustatory sense. \emph{Sensory prediction} means predicting human sensory judgments from product information. Our food-level task predicts panelist-rated \emph{similarity} to an animal-based reference, a holistic judgment distinct from liking.} Could models infer how a food will taste? The development of sustainable proteins, including plant-based, fermentation-enabled, and cultivated products, provides a consequential testbed for this question, as diversifying protein sources away from animal agriculture is among the highest leverage interventions for climate change mitigation~\citep{ipcc2019land}. While drug and materials discovery have been transformed by fast computational proxies - molecular
docking, density functional theory - that allow researchers to screen candidates 
\textit{in silico} before committing to costly experiments~\citep{fan2019molecular, jain2016computational}, sustainable protein development lacks a comparably established proxy for human sensory evaluation. Today, evaluating whether a sustainable protein product tastes like its animal-based target requires recruiting large panels of human tasters, a process that is slow, expensive, and inherently low-throughput. 
Moreover, 46\% of consumers cite poor taste as a barrier to repeat purchase, and large-scale benchmarking conducted by the nonprofit organization NECTAR has documented persistent taste gaps between sustainable proteins and their animal-based counterparts \citep{nectar2025taste}.

The sustainable protein formulation space is enormous, as varying protein sources, binding agents, fats,
flavors, and processing conditions creates a combinatorial design problem, but the
evaluation bottleneck means that only a handful of candidates can be panel-tested per
development cycle. Just as molecular docking ranks drug candidates by predicted
binding affinity so that only top-scoring compounds advance to cell-based assays, a
sensory predictor could rank candidate formulations by predicted similarity to the
target animal product, focusing expensive panel testing on the most promising
candidates. An AI system that could reliably perform this ranking would accelerate the design-build-test loop. More ambitiously, such a system could participate in closed-loop formulation optimization, in which candidates are iteratively proposed and refined, potentially guiding robotic chefs and self-driving food laboratories toward formulations with desired human sensory responses.

Prior work has begun to explore this direction.~\citet{thomas2025what} evaluated
several LLMs on sensory prediction tasks constructed from an earlier version of the
NECTAR dataset, but considered only text-based inputs and did not attempt to improve
upon foundation model performance. Separately, recent work on molecular-level taste
and olfaction prediction has produced capable models for classifying individual
compounds \citep{zimmermann2025chemical, lee2023principal}, but no prior work has
explored whether such molecular representations transfer to food-level
sensory prediction. More broadly, no standardized benchmark exists for this task.
Progress has been difficult to measure in part because sensory data carries
commercial sensitivity: companies that participate in panel testing are reluctant to
share results that could identify underperforming products. Building useful
evaluation infrastructure for this domain therefore requires not only curating tasks
and baselines, but also designing data-sharing mechanisms that preserve brand
privacy. Our contributions include:

\begin{custombox}{Contributions}{D2Green}
\begin{itemize}
\item \key{A multimodal, multiresolution benchmark and privacy-preserving competition for sensory prediction.} We introduce \textsc{TasteBench}, spanning two tasks: a food-level ranking task (21K+ evaluations across 215 products, yielding 935 pairs) and a molecular-level taste classification task (15K+ molecules). The benchmark integrates four feature cost tiers - ingredient text, nutrition, molecular composition, and appearance - enabling systematic study of the cost-accuracy tradeoff. Both tasks are released as public Kaggle competitions, using decoy data to prevent brand-identity inference.
\item \key{Ground truth characterization for ML evaluation.} We quantify the reliability of the human judgments against which models are scored. The median panelist achieves .650 pairwise ranking accuracy against the panel mean, inter-rater agreement is low ($\alpha = .077$), and the split-half reliability ceiling is .825. This establishes the interpretive range for model performance: scores near .650 are indistinguishable from a single panelist, while scores approaching .825 hit the noise floor of the ground truth. This confirms that panel-aggregated targets are the appropriate ground truth for ML evaluation in this domain.
\item \key{Computational baselines and empirical findings.} We evaluate unsupervised distance methods, zero-shot foundation models, supervised pairwise models, and ensembles. On within-block pairs, the best model achieves .661 accuracy, competitive with the median panelist; on all within-category pairs, the best model achieves .683. Furthermore, three molecular encoders -~\citet{zimmermann2025chemical}'s Flavor Analysis and Recognition Transformer (FART), a from-scratch D-MPNN, and a fine-tuned ChemBERTa-2 - span .84--.90 molecular-classification accuracy yet reach statistically indistinguishable food-level ranking accuracy (${\approx}.68$), suggesting that directional improvements on a component-level ML task do not strictly transfer to the system-level ML task the benchmark is designed to measure.
\end{itemize}
\end{custombox}

By establishing this seed benchmark, \benchname unlocks several critical research questions for the community:

\begin{custombox}{Research Questions}{D2Blue}
\begin{itemize}
\item \key{Multimodal sensory integration.} How can models effectively fuse visual, nutritional, and structural chemical data to predict noisy, subjective human sensory responses?
\item \key{The chemistry/semantic gap.} Why do text-based LLMs currently outperform dedicated molecular graphs in macroscopic food matrices? How can we build multi-sensory grounded models that capture missing modalities (like texture and olfaction) without merely relying on memorized semantic correlations?
 \item \key{Data augmentation in low-data regimes.} How can the community utilize this calibrated seed benchmark to bridge small, high-fidelity psychophysical datasets with massive, unstructured web data (e.g., recipes, photos and reviews) to accelerate discovery?
\end{itemize}
\end{custombox}

We further discuss related work in Appendix~\ref{app:related-work}.

\section{Datasets}\label{sec:datasets}

\textsc{TasteBench} integrates four data sources to  map molecular, text, nutritional, and visual features to human sensory outcomes while preserving privacy.

\mypar{NECTAR} ~\citep{nectar2025taste,nectar2026taste} consists of 271 sustainable protein products and animal-based benchmarks and at least 100 sensory ratings per product, along the dimensions of overall liking, similarity to the target animal product, flavor, texture, appearance, and purchase intent. Here we use data published in NECTAR's 2025 Taste of the Industry report, consisting of 121 plant-based meats across 14 categories evaluated by 2,684 omnivores, and NECTAR's 2026 Taste of the Industry report, consisting of 94 dairy-free products across 10 categories evaluated by 2,183 omnivores and vegetarians. Available features include ingredient lists, nutritional information, and photos of each food item. The public interactive NECTAR data dashboard is \href{https://dashboard.palateinsights.com/taste-of-the-industry}{here}. To protect the privacy of participating brands, the specific product identities (e.g., brand names) have been de-identified, though the sensory targets, evaluation suite, and underlying product features needed to run the \textsc{TasteBench} benchmark are fully public. Access to the unblinded product identities can be requested  \href{https://docs.google.com/forms/d/e/1FAIpQLSdYXkB1yEOt_gMBrnHIpHv_xX9peh5NS1kdbfuY8-6-77kSug/viewform?usp=preview}{here}.  Our models were evaluated on 215 plant-based and vegetarian meat and dairy analogs across 24 categories. This creates 935 within-category ranking pairs.\footnote{Three of the 935 pairs have tied panel-mean similarity scores; they are excluded from supervised training and scored as 0.5 in pairwise-accuracy evaluation.} An additional 32 animal-based products were included as references for feature development. This resulted in a total of 247 products in our dataset. Data collection for 134 additional products will be completed by the end of 2026, and will continue in future years.

\mypar{Taste Like}\label{sec:taste-like}~\citep{nectarNECTARAcquires} is a directory of over
1,200 plant-based and fermentation-based products from 147 companies, with
ingredient lists, nutrition facts, and allergen metadata. In
\textsc{TasteBench}, these products serve a structural role: they anonymize
the NECTAR-evaluated products in the public Kaggle competition. Because
NECTAR's TASTY Awards publicly identify top performers, releasing features
for only NECTAR products would allow for inferring which
products scored poorly. Including Taste Like products (which lack sensory
ratings) prevents this inference. Taste Like products do not affect scoring.


\mypar{FoodAtlas} ~\citep{foodAtlas} is a knowledge graph containing 285{,}077 triplets of three entity types (\textit{food}, \textit{part}, \textit{chemical}) and four relation types (\textit{contains}, \textit{has part}, \textit{is a}, \textit{has child}) on three evidence quality levels (\textit{high}, \textit{medium}, \textit{low}), with 4{,}318 of them evaluated by human experts. We use FoodAtlas to expand each ingredient on a NECTAR product label into a list of associated compounds together with their concentrations where reported; each compound's identifier is then mapped to a canonical SMILES which can then be featurized.

\mypar{FartDB} ~\citep{zimmermann2025chemical} (Flavor Analysis and Recognition Transformer)\footnote{We use the GitHub release at \url{https://github.com/fart-lab/fart/tree/main/dataset}; a HuggingFace mirror is available at \url{https://huggingface.co/datasets/FartLabs/FartDB}.} is a publicly available collection of 15{,}025 canonicalized small molecules labeled with one of five taste classes: \textit{sweet}, \textit{bitter}, \textit{sour}, \textit{umami}, and \textit{undefined}. The dataset is curated from six public sources: FlavorDB~\citep{garg2017flavordb}, PlantMolecularTasteDB~\citep{gradinaru2022plantmolecular}, ChemTastesDB~\citep{rojas2021chemtastesdb}, Tas2R Agonists DB~\citep{bayer2021tas2r}, IUPAC Digitized Dissociation Constants~\citep{zheng2022iupac}, and Umami Compounds~\citep{suess2015umami}.

Molecules with multiple taste labels appear as separate entries. We use the original train/validation/test splits provided by the authors (10{,}517\,/\,2{,}254\,/\,2{,}254 molecules). The dataset is heavily skewed toward the sweet class, with umami severely underrepresented (53 training molecules, 6 test molecules). Our models are evaluated on the 2{,}254-molecule held-out test split for the molecular prediction task.

License terms and access restrictions for all datasets are documented in Appendix~\ref{app:dataset-licenses}. Croissant metadata for the NECTAR and Taste Like datasets, with Responsible AI fields, is available \href{https://github.com/Food-Intelligence-Lab/tastebench/blob/main/croissant/nectar.json}{here}.

\section{Benchmark Design}\label{sec:benchmark-design}
We evaluate predictive capabilities across two distinct scales: a macroscopic food-level ranking task that simulates human sensory panels, and a microscopic molecular classification task to test the limits of chemical proxies.

\begin{figure}[h!]
\centering
\includegraphics[width=\textwidth]{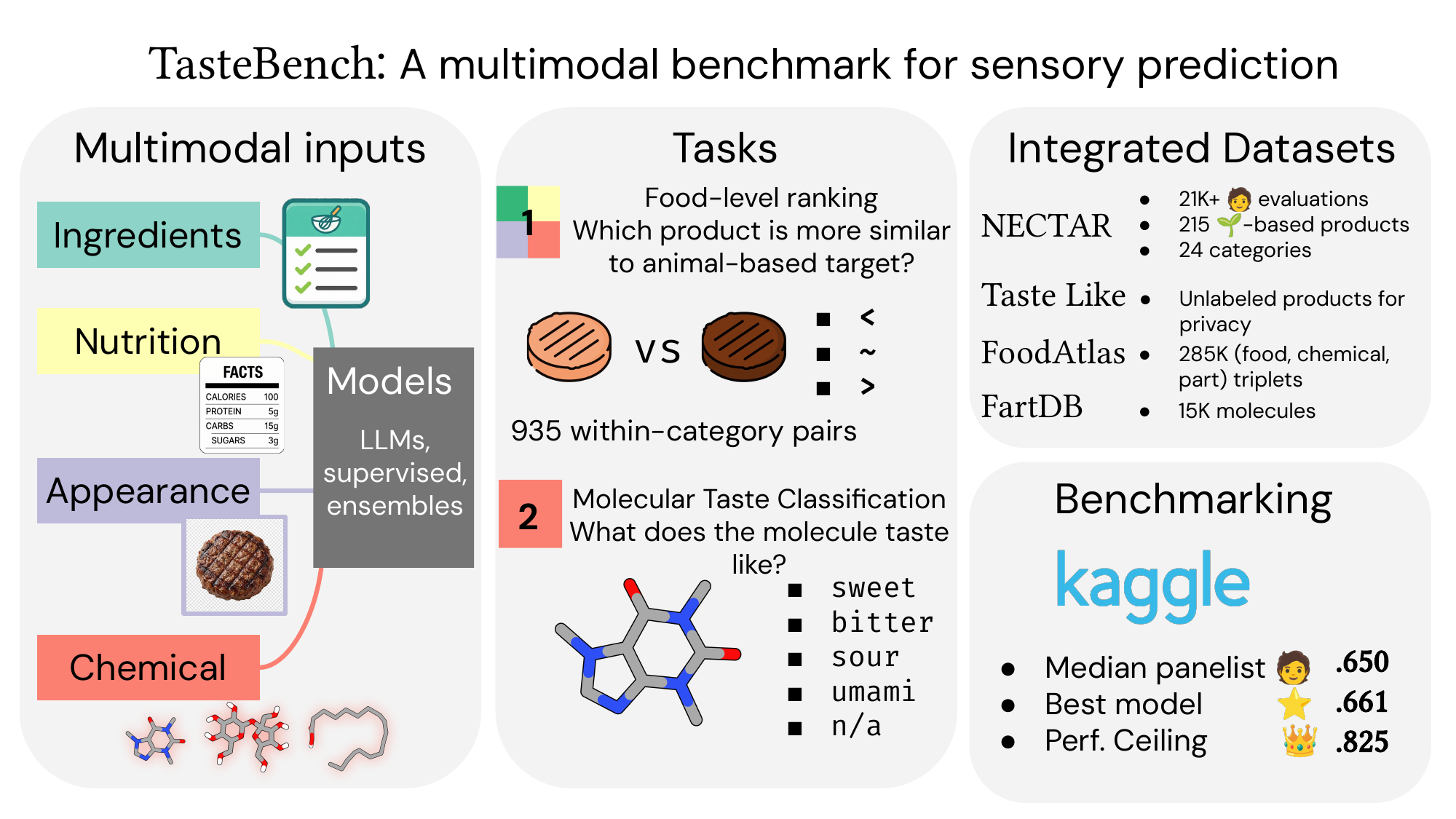}
\caption{Overview of \textsc{TasteBench}, a multimodal benchmark built on 21k+ human evaluations across 215 plant-based foods in 24 product categories, and 15K flavor molecules.}
\label{fig:main}
\end{figure}

\paragraph{Food-Level Similarity Prediction Task}
While NECTAR contains multiple sensory dimensions, here we focus on predicting similarity to the target animal product, a key driver of overall liking~\citep{giezenaar2024profiling}. Specifically, we frame the problem as a within-category ranking task: given a set of plant-based products in the same category, rank them by their mean panelist-rated similarity to the animal-based counterpart. We use multiple metrics to evaluate model performance: (1) Pairwise ranking accuracy; (2) Spearman $\rho$ and Kendall $\tau$ rank correlation (weighted by number of products per category); (3) Recall@$k$ ($k = 1, 2, 3$): fraction of categories where the true best product is in the top $k$ predictions (macro-averaged).

Our features span multiple cost tiers (Appendix Figure~\ref{fig:cost-hierarchy}): ingredient text, nutrition, and molecular composition require no physical product, while appearance requires making and photographing the product. Ground-truth sensory panels are the most expensive tier, ranging from \$2,100 for a 20-person university panel~\citep{uarizona_sensory} to \$18,000--\$85,000 for outsourced consumer tests~\citep{dataintelo2026sensory}.

\paragraph{Molecular-Level Taste Classification Task}

Given a canonical SMILES string, we predict whether the molecule is sweet, bitter, sour, umami, or undefined in this task. This is the same classification task performed by the chemical language model, FART (Flavor Analysis and Recognition Transformer) ~\citep{zimmermann2025chemical}. We extend their contribution by evaluating a different model (D-MPNN;~\cite{heid2023chemprop}) we train on the same data splits to understand how embeddings generated from different molecular taste classification models with comparable performance transfer to performance on food-level similarity predictions.

\paragraph{Privacy-Preserving Competition}\label{sec:privacy-preserving-competition}


To publicly compare the performance of different approaches on the food-level similarity prediction task, we create a test-only public Kaggle competition. Input features for the sustainable proteins are publicly available, and performance on the hidden test set is computed once a competitor uploads their predictions. A key consideration in this setting is preserving the privacy of companies whose products performed poorly in the NECTAR sensory testing. Thus, solely including features of NECTAR products would, in conjunction with the results of the TASTY Awards (given to the top performing products) enable identification of the brands that didn't win a TASTY Award.
We reduce this risk via also incorporating samples from the Taste Like dataset of sustainable proteins~\citep{nectarNECTARAcquires}, thus anonymizing the NECTAR products. The Kaggle competition is available \href{https://www.kaggle.com/competitions/tastebench-challenge-2026/}{here}. We discuss defenses against adversarial attacks in Appendix~\ref{app:privacy-details}. Similarly, to publicly compare the performance of different approaches on the molecular-level taste classification task, we create a fully public Kaggle competition, available \href{https://www.kaggle.com/competitions/molecular-taste-classification-2026/}{here}. This is the FartDB benchmark ~\citep{zimmermann2025chemical}, released on Kaggle as a public, transparent leaderboard. 
\paragraph{Access Structure and Reproducibility}\label{sec:access-structure}
\textsc{TasteBench} separates benchmark use from brand-identifying information. Both Kaggle competitions are public: participants can download input features, submit predictions, and receive scores without access to product identities. We also release the code, trained model weights, and other files needed to reproduce all reported results. Released model weights, out-of-fold predictions, cached features, and reproduction scripts are inventoried in Appendix~\ref{app:reproducibility}. The only controlled-access component is the linkage between de-identified product codes and underlying NECTAR product records, including brand identities. This linkage is not required to use the benchmark, submit to either competition, evaluate predictions, or reproduce results from the released feature files. It is needed only to audit the de-identification procedure or rebuild the benchmark from raw NECTAR records. Because participating companies provided data under anonymity constraints, this linkage is restricted. 









\section{Baseline Models}\label{sec:baseline-models}

To predict food-level similarity, we extract features across four modalities: nutrition, appearance, ingredient text, and molecular composition. Because molecular chemistry is the most complex modality, we evaluate it as a standalone sub-task (Molecular Taste Classification) to determine if better chemical classifiers yield better downstream food-level representations.

\subsection{Features \& the Molecular Taste Sub-Task}\label{sec:food-level-features}

\paragraph{Nutrition}
We extend the nutrition features of~\citet{thomas2025what}, who used normalized fat and sodium, to a set of six macronutrients: total fat, sodium, protein, dietary fiber, total carbohydrate, and total sugars. These features were selected based on literature linking the nutrient and sensory properties in each relevant product type (Appendix~\ref{app:nutrition-selection}). All values are normalized to per-100g serving and converted to grams. In our supervised models, all six columns are provided to every product and the model learns which are relevant per category; missing values are imputed with zero. For the unsupervised distance baselines, only the category-specific subset is used. LLM-based models receive the full nutrition facts panel as formatted text in the prompt.

\paragraph{Ingredient Text}

To generate a text embedding of the ingredients, we embed the full ingredient list as a single string using Qwen3-Embedding-0.6B. We prepend a task-specific instruction (''Given a food ingredient list, identify products with similar taste, texture, and sensory properties'') following the model's instruct format. This produces a 1024-dimensional food-level embedding.

\paragraph{Category Subset}

We group the 24 NECTAR product categories into four subsets that share sensory-relevant macronutrients: \textit{meat}, \textit{nonsweet dairy}, \textit{cheese}, and \textit{sweet dairy} (full mapping in Appendix~\ref{app:category-subsets}). Per-category sample sizes are too small for category-specific modeling (as few as 5 products per category), so each product is encoded as one dimension of a 4-dimensional one-hot indicator over these subsets. In our supervised models, this lets the model learn product-type-specific coefficients without per-category overfitting; in our unsupervised distance baselines, the corresponding subset additionally selects which macronutrient features to use (Appendix~\ref{app:nutrition-selection}).

\paragraph{Appearance}
To generate an appearance embedding for each image, we use the DINOv3 vision transformer model \citep{oquab2024dinov}. We extract the CLS token from the last hidden state to generate 1024-dimensional embeddings. 93\% of analog products have an image available; for the rest, missing image embeddings are imputed from the $k{=}5$ nearest neighbors within the same category, where neighbor distance is computed via cosine similarity over the available (non-image) features.

\paragraph{Chemical Composition \& Molecular Taste Classification Task}

We extract the molecular composition of each ingredient using the FoodAtlas knowledge graph, resolving compounds to canonical SMILES via PubChem or ChEBI. We evaluate two distinct methods for embedding these molecules. First, we default to extracting 768-dimensional embeddings from the CLS token of the FART chemical language model \citep{zimmermann2025chemical}. These compound-level embeddings are aggregated to the food level via a log-concentration-weighted mean (weights $w_i \propto \log(1 + c_i)$). While complex attention-based aggregators are common~\citep{NEURIPS2025_2f803abd}, recent chemosensory benchmarks demonstrate that in high-noise psychophysical datasets~\citep{Satarifard2025-lf}, simple log-linear aggregations resist catastrophic overfitting and effectively mimic biological receptor competitive masking, aligning with the Weber-Fechner law~\citep{DEHAENE2003145}. Ingredient-level embeddings are then averaged across the top three ingredients by FDA label weight.

To determine if improving molecular taste classification yields better downstream food-level rankings, we evaluate a standalone sub-task by training a directed message-passing network (D-MPNN; \citealt{heid2023chemprop}) from scratch. We parse SMILES into graphs using Chemprop’s default atom and bond featurizer, and train the D-MPNN on the same five-class taste classification task (sweet, bitter, sour, umami, undefined) as FART. The model mean-pools atom representations into a 300-dimensional graph embedding. Details are in Appendix \ref{app:gnn-training}. This trained GNN establishes our molecular-task baseline (Table~\ref{tab:molecular-prediction}) and provides alternative 300-dimensional embeddings that we swap with FART to test cross-resolution transfer in our downstream product models (Table~\ref{tab:encoder-transfer}). We also add a third encoder to both comparisons, a ChemBERTa-2 chemical language model fine-tuned on the same split.

\subsection{Food Similarity Models}
\label{sec:models_food_sim}

After generating features, we use the models described below to predict the mean similarity rating for each plant-based protein to its animal-based counterpart.

\mypar{Multi-modal Rank Fusion (MMRF)} For each product category, we first compute the mean feature vector of animal-based reference products per modality (nutrition, compound, image). We then compute each plant-based product's distance to its category's reference centroid, normalize via within-category percentile ranks, and average ranks across modalities to produce a final similarity ranking within each category. We evaluate both cosine and Euclidean distance variants.

\mypar{LLMs}
We tested Gemini 3.1 Pro and Qwen 3.5 397B-A17B, leading closed-source and open-source multimodal foundation models, to understand how they perform with zero-shot prompting. The prompt is in Appendix~\ref{app:llm-prompts}. Ingredients, nutrition, and images are used as inputs in the prompts.

\mypar{Supervised Models} We train supervised models via leave-one-product-out cross-validation (LOOCV) across all 215 plant-based products. All supervised models receive the same concatenated feature vector (product category subset, nutrition proportions, compound aggregations, text embeddings, and image embeddings) with per-modality standardization and dimensions reduced via PCA to 95\% explained variance for the high-dimensional modalities (both fit on training data within each LOOCV fold). We evaluate five models: ridge regression; Bradley-Terry, which fits logistic regression on within-category feature-pair differences; Hierarchical Bradley-Terry, which fits per-category-subset models with empirical Bayes shrinkage toward the global coefficients; Kernel RankSVM with an RBF kernel on pair differences; and LightGBM with a pointwise MSE objective.

\mypar{Ensembles} We combine the best individual model (Gemini 3.1 Pro) with the supervised model (Bradley-Terry) that yields the most improvement in accuracy; Appendix~\ref{app:architecture-diagram} traces a single product through this full pipeline end-to-end. A comparison of the performance improvement of each supervised model when combined with Gemini can be seen in Appendix~\ref{app:per-model-nnls}.

We use nested leave-one-out cross-validation to prevent label leakage: for each held-out product, inner-loop Bradley-Terry predictions are trained excluding both the held-out product and the evaluation partner. Our results report three combination strategies: non-negative least squares (NNLS), rank averaging, and arithmetic mean.

\section{Experiments and Results}

\subsection{Ground truth reliability and human performance}\label{sec:human-baseline}

To interpret model performance, we benchmark against human panelists,
following~\cite{lee2023principal}. NECTAR recruits untrained consumer panelists using
a balanced incomplete block design: each panelist rates a subset of products within a
category per session, across two sessions on separate days. We compute each panelist's pairwise ranking accuracy against the leave-one-out panel
mean, restricted to the products they rated. The median panelist achieves .650
accuracy (IQR .50--.75, $n = 627$), with low inter-rater agreement (Krippendorff's
$\alpha = .077$). To compare on the same pair set, we evaluate our best model (described in Section~\ref{sec:ensembles-exceed}) on
within-block pairs, where it achieves $.661$, exceeding the median panelist by
$1.1$\,pp; on all within-category pairs (Table~\ref{tab:results}), it achieves
$.683$, matching approximately $k^* \approx 2$ panelists ($p < .001$).
Figure~\ref{fig:human-baseline} shows accuracy as a function of panel size. The
split-half reliability of the panel mean ranking ($.825$) represents an approximate
ceiling for model performance.

\begin{figure}[h!]
\centering
\begin{minipage}{0.58\columnwidth}
  \includegraphics[width=\linewidth]{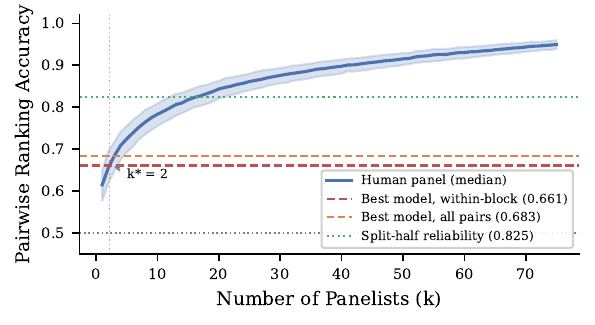}
\end{minipage}\hfill
\begin{minipage}{0.40\columnwidth}
  \scriptsize
  \captionof{table}{\textbf{Human-panelist baseline.} Pairwise ranking accuracy against the leave-one-out panel mean on within-block pairs. The meaningful band for models runs from .650 (median panelist) to .825 (split-half reliability ceiling). $n$ = 627 panelists, 4{,}664 panelist-block observations.}
\label{tab:human-baseline}
\begin{tabular}{@{}lc@{}}
\toprule
Method & Pairwise Acc.\ [95\% CI] \\
\midrule
\multicolumn{2}{@{}l}{\textit{Evaluated on within-block pairs}} \\
\quad Individual panelist (median) & .650\textsuperscript{\dag} \\
\quad Panel of 3 & .685 [.640,.725] \\
\quad Panel of 5 & .725 [.686,.760] \\
\quad Best model (within-block) & .661 [.604,.716] \\
\midrule
\multicolumn{2}{@{}l}{\textit{Evaluated on all within-category pairs}} \\
\quad Best model (all pairs) & .683 [.653,.740] \\
\bottomrule
\end{tabular}

\vspace{2pt}
{\scriptsize \textsuperscript{\dag}IQR: 0.50--0.75 across 4{,}664 panelist-block observations.}

\end{minipage}
\mycaption{Models perform competitively with small human panels}{Pairwise accuracy vs.\ panel size (left) and discrete comparison (right, Table~\ref{tab:human-baseline}).}
\label{fig:human-baseline}
\end{figure}

\subsection{Ensembles exceed median human similarity matching performance}\label{sec:ensembles-exceed}

The results for our computational baselines are shown in Table~\ref{tab:results}. We report 95\% BCa bootstrap confidence intervals (10{,}000 resamples, stratified by category)~\citep{efron1987better}. The best ensemble (NNLS over Bradley-Terry and Gemini 3.1 Pro) achieves $.683$ pairwise accuracy [$.654$, $.740$], exceeding both the strongest individual model (Gemini 3.1 Pro at $.654$) and the median untrained human panelist ($.650$); this corresponds to matching approximately $k^{*} \approx 2$ panelists (Section~\ref{sec:human-baseline}).

Among individual models, the leading LLM (Gemini 3.1 Pro, $.654$) and the strongest supervised model (Kernel RankSVM, $.618$) are within their CIs of each other, while distance baselines ($.571$--$.574$) and LightGBM ($.558$) lag substantially. NNLS exceeds either component alone ($.683$ vs.\ $.654$ / $.610$), indicating that the LLM and supervised model capture complementary signal; rank-averaging and arithmetic-mean variants recover most but not all of this gain ($.671$ and $.623$ respectively).

Hierarchical Bradley-Terry recovers the true best product in its top-3 predictions in $67\%$ of categories (R@3 = .667) and predicts the true top product in $33\%$ of categories (R@1 = .333 vs.\ .121 random). Per-modality and per-feature ablations are reported in Appendix~\ref{app:add-results}. Because category sizes are uneven (5--18 products), we also report macro-averaged and partial-pooling estimates in Appendix~\ref{app:category-balanced}; the best model ranks first under all three schemes.

Many resource-constrained brands rely on feedback from 3--5 untrained internal participants rather than formal consumer panels~\citep{nolden2025sensory}. A model competitive with the median untrained panelist could provide a low-cost screening signal where infrastructure for formal panels is unavailable.

\begin{table}[t]
\centering
\scriptsize
\setlength{\tabcolsep}{2pt}
\renewcommand{\arraystretch}{0.92}
\caption{
\textbf{Ensembles of supervised models and LLMs outperform the median human panelist.}
Ranking metrics on 215 NECTAR plant-based products (935 within-category pairs; LOOCV) with 95\% BCa CIs. \textbf{Bold} = best in column; $^\dag$ = CI overlaps the leader (no significant difference at 95\%)
}
\label{tab:results}
\resizebox{\textwidth}{!}{%
\begin{tabular}{@{}lcccccc@{}}
\toprule
Model & Pw.\ Acc. & $\rho_{\mathrm{S}}$ & $\tau_{\mathrm{K}}$ & R@1 & R@2 & R@3 \\
\midrule
Gemini 3.1 Pro & .654$^\dag$ [.619,.713] & .370$^\dag$ [.250,.509] & .298$^\dag$ [.179,.409] & .271$^\dag$ [.167,.424] & \textbf{.514} [.476,.701] & .578$^\dag$ [.491,.786] \\
Qwen 3.5 397B & .630$^\dag$ [.591,.688] & .320$^\dag$ [.203,.459] & .255$^\dag$ [.141,.365] & .229$^\dag$ [.132,.354] & .354$^\dag$ [.215,.483] & .500$^\dag$ [.375,.656] \\
MMRF (cosine) & .574 [.528,.635] & .212$^\dag$ [.077,.380] & .151$^\dag$ [.023,.278] & .208$^\dag$ [.104,.323] & .292 [.135,.396] & .333 [.181,.379] \\
MMRF (L2) & .571 [.524,.632] & .202$^\dag$ [.064,.358] & .153$^\dag$ [.031,.277] & .222$^\dag$ [.104,.335] & .403$^\dag$ [.285,.576] & .542$^\dag$ [.438,.729] \\
\midrule
Ridge & .579 [.534,.632] & .160 [.028,.302] & .115$^\dag$ [$-.005$,.226] & .250$^\dag$ [.181,.435] & .292 [.118,.396] & .542$^\dag$ [.493,.749] \\
Bradley--Terry & .610$^\dag$ [.571,.661] & .266$^\dag$ [.156,.402] & .178$^\dag$ [.064,.287] & .250$^\dag$ [.181,.422] & .292 [.125,.396] & .375 [.187,.458] \\
Hierarchical BT & .613$^\dag$ [.570,.670] & .228$^\dag$ [.098,.383] & .165$^\dag$ [.043,.283] & \textbf{.333} [.271,.526] & .500$^\dag$ [.424,.742] & \textbf{.667} [.622,.854] \\
Kernel RankSVM & .618$^\dag$ [.579,.674] & .288$^\dag$ [.168,.438] & .206$^\dag$ [.088,.321] & .167 [.040,.250] & .292 [.115,.396] & .417 [.196,.521] \\
LightGBM & .558 [.509,.615] & .136 [$-.003$,.293] & .092$^\dag$ [$-.032$,.219] & .292$^\dag$ [.250,.438] & .333$^\dag$ [.208,.491] & .458$^\dag$ [.333,.646] \\
\midrule
NNLS (BT+Gemini) & \textbf{.683} [.654,.740] & \textbf{.417} [.314,.550] & \textbf{.322} [.211,.427] & .292$^\dag$ [.188,.451] & .500$^\dag$ [.424,.693] & .542$^\dag$ [.385,.688] \\
Rank avg. & .671$^\dag$ [.641,.725] & .398$^\dag$ [.293,.537] & .304$^\dag$ [.192,.409] & .264$^\dag$ [.182,.427] & .361$^\dag$ [.211,.486] & .521$^\dag$ [.375,.677] \\
Mean & .623$^\dag$ [.587,.675] & .310$^\dag$ [.201,.443] & .216$^\dag$ [.102,.321] & .167 [.035,.250] & .375$^\dag$ [.260,.549] & .417 [.219,.500] \\
\midrule
Random (analytical) & .500 & .000 & .000 & .121 & .242 & .363 \\
\bottomrule
\end{tabular}%
}
\end{table}

\subsection{Molecular-classification quality does not transfer to food-level ranking}\label{sec:transfer}


\begin{table}[t]
\centering
\scriptsize
\setlength{\tabcolsep}{3pt}
\renewcommand{\arraystretch}{0.95}
\caption{Taste classification on the FART test split (2{,}254 molecules). Baselines from~\citet{zimmermann2025chemical}; 95\% BCa CIs (10{,}000 resamples). \textbf{Bold} = best; $^\dag$ = CI overlaps leader. Per-class breakdown in Table~\ref{tab:molecular-prediction-per-class}.}
\label{tab:molecular-prediction}
\begin{tabular}{@{}lccccc@{}}
\toprule
Model & Accuracy & Precision & Recall & F1 & AUROC \\
\midrule
\multicolumn{6}{@{}l}{\textit{Reported by~\citet{zimmermann2025chemical}}} \\
\quad XGBoost (fp) & .899 & .817 & .740 & .766 & .852 \\
\quad XGBoost (fp+desc) & .896 & \textbf{.884} & .740 & .778 & .851 \\
\quad Balanced RF (fp) & .797 & .585 & .732 & .601 & .839 \\
\quad FART & .894$^\dag$ [.880,.906] & .879 [.675,.897] & .739$^\dag$ [.668,.869] & .774$^\dag$ [.671,.875] & .974$^\dag$ [.968,.979] \\
\midrule
\multicolumn{6}{@{}l}{\textit{This work}} \\
\quad GNN & \textbf{.904} [.890,.915] & .825 [.697,.895] & \textbf{.813} [.683,.883] & \textbf{.818} [.732,.876] & \textbf{.978} [.968,.982] \\
\quad ChemBERTa-2 (fine-tuned) & .842 [.827,.857] & .643 [.590,.709] & .710$^\dag$ [.635,.802] & .671 [.609,.731] & .950 \\
\bottomrule
\end{tabular}
\end{table}

On the held-out FART test split (2{,}254 molecules), the GNN achieves an accuracy of $.904$ [$.890$, $.915$] compared to FART's $.894$ [$.880$, $.906$]. The gap is not significant (McNemar $p{=}0.07$). A third encoder, ChemBERTa-2, fine-tuned on the same split scores $.842$ [.827,.857], significantly below both (per-class breakdown in Appendix~\ref{app:molecular-per-class}).

\textbf{Transferring molecular representations to food-level ranking.} We replace the compound block with each encoder's embeddings across all non-LLM models, holding everything else fixed. As shown in Table~\ref{tab:encoder-transfer}, downstream pairwise accuracy is statistically indistinguishable across the three encoders, which span .84--.90 on the molecular task. This decoupling suggests that molecular taste classification is a limited proxy for the compound interactions that drive product-level preferences. We use FART in the main pipeline because it requires no training.

\begin{table}[t]
\centering
\footnotesize
\setlength{\tabcolsep}{4pt}
\renewcommand{\arraystretch}{0.95}
\caption{\textbf{Molecular-encoder quality does not transfer to product-level ranking.} Pairwise accuracy (935 within-category pairs, LOOCV), compound block swapped, all else fixed. 95\% BCa CIs (10{,}000 resamples). $\Delta$ relative to FART; all lie inside their CIs.}
\label{tab:encoder-transfer}
\begin{tabular}{lccccc}
\toprule
Model & C $=$ FART & C $=$ GNN & C $=$ ChemBERTa-2 & $\Delta_{\text{GNN}}$ & $\Delta_{\text{Chem}}$ \\
\midrule
\multicolumn{6}{@{}l}{\emph{Unsupervised}} \\
\quad MMRF (cosine) & \textbf{.574} {\scriptsize [.528,.635]} & .566 {\scriptsize [.518,.624]} & .573 {\scriptsize [.526,.633]} & $-$.008 & $-$.002 \\
\quad MMRF (L2) & \textbf{.571} {\scriptsize [.524,.632]} & .563 {\scriptsize [.518,.619]} & .568 {\scriptsize [.522,.628]} & $-$.008 & $-$.002 \\
\multicolumn{6}{@{}l}{\emph{Supervised -- linear}} \\
\quad Ridge & .579 {\scriptsize [.534,.632]} & .561 {\scriptsize [.513,.614]} & \textbf{.598} {\scriptsize [.554,.655]} & $-$.018 & $+$.019 \\
\multicolumn{6}{@{}l}{\emph{Supervised -- pairwise}} \\
\quad Bradley--Terry & .610 {\scriptsize [.571,.661]} & .617 {\scriptsize [.578,.671]} & \textbf{.640} {\scriptsize [.604,.693]} & $+$.006 & $+$.030 \\
\quad Hierarchical BT & .613 {\scriptsize [.570,.670]} & .630 {\scriptsize [.588,.689]} & \textbf{.639} {\scriptsize [.599,.698]} & $+$.017 & $+$.026 \\
\quad Kernel RankSVM & .618 {\scriptsize [.579,.674]} & \textbf{.627} {\scriptsize [.589,.682]} & .610 {\scriptsize [.570,.664]} & $+$.010 & $-$.007 \\
\multicolumn{6}{@{}l}{\emph{Supervised -- nonlinear}} \\
\quad LightGBM & \textbf{.558} {\scriptsize [.509,.615]} & .515 {\scriptsize [.459,.571]} & .535 {\scriptsize [.483,.592]} & $-$.043 & $-$.022 \\
\multicolumn{6}{@{}l}{\emph{Ensemble (BT + Gemini)}} \\
\quad NNLS & .683 {\scriptsize [.654,.740]} & .673 {\scriptsize [.641,.730]} & \textbf{.684} {\scriptsize [.655,.737]} & $-$.010 & $+$.001 \\
\quad Rank average & .671 {\scriptsize [.641,.725]} & .671 {\scriptsize [.639,.725]} & \textbf{.685} {\scriptsize [.656,.743]} & $+$.000 & $+$.014 \\
\quad Mean & .623 {\scriptsize [.587,.675]} & .635 {\scriptsize [.599,.691]} & \textbf{.641} {\scriptsize [.606,.695]} & $+$.012 & $+$.018 \\
\bottomrule
\end{tabular}
\end{table}

\section{Discussion}\label{sec:discussion}

\textsc{TasteBench} provides evaluation infrastructure for predicting how sustainable
proteins will perform in sensory evaluation. Our baselines establish
that current models can exceed the median untrained panelist on pairwise ranking,
while the reliability ceiling (.825) provides an upper bound for future
progress. Even at current accuracy, these models are practically useful as screening tools (Appendix~\ref{sec:theoretical}). In our molecular prediction experiments, three encoders spanning six points of
molecular-classification accuracy give statistically indistinguishable food-level
ranking, suggesting that improving molecular classifiers alone will not close the
performance gap. 

\paragraph{The modality gap and data contamination.} Our baselines reveal that zero-shot multi-modal LLMs often outperform dedicated molecular representations. However, this advantage might stem from ``data contamination": LLMs having ingested public CPG databases and recipe co-occurrences during pre-training. We audit this in Appendix~\ref{app:recognition-audit}. The model's reasoning identifies specific products in about 12\% of pairs without higher accuracy, and an interventional brand-blind control changes accuracy by only $-.005$ ($p = .71$), bounding the latent contribution. Because text and images cannot fully capture physical texture, taste or olfaction, advancing chemical mixture models remains essential to overcome this data scarcity and discover truly out-of-distribution formulations.

\paragraph{A seed benchmark for the community.} By establishing the human reliability ceiling, \benchname provides a calibrated framework to measure algorithmic progress and invites the community to augment these tasks with large-scale web data to further close the gap to human perception. 

\paragraph{Societal impact.}
Sensory prediction could accelerate sustainable protein development, with downstream
benefits for emissions, land use, and water use. However, biased models could narrow
the formulation design space, and overreliance on screening before models are
sufficiently reliable could lead to premature elimination of promising candidates.
The human performance baseline and reliability ceiling serve as partial safeguards
against overconfidence.

\paragraph{Limitations.}
The benchmark covers meat and dairy only, though will be updated as NECTAR expands to eggs and seafood. The panel is a screened, two-city US consumer sample rather than a
nationally representative one; generalization to other markets or to trained panelists
is untested. Per-category sample sizes are
small (median 36 pairs, but as few as 10), producing wide category-level confidence
intervals; primary metrics are computed across all 935 within-category pairs. Our
evaluation focuses on the similarity dimension, and other sensory dimensions may require
different modeling approaches. We have not yet established an expert food scientist
baseline or explored prompt engineering for the foundation models.

   \begin{ack}
    Thank you to the Bezos Earth Fund, Food System Innovations, and Stanford HAI for supporting this work.
    \end{ack}

\bibliography{example_paper}
\bibliographystyle{icml2026}

\appendix
\raggedbottom
\clearpage 

\section{Related Work}\label{app:related-work}

\paragraph{Taste Prediction.} \citet{zimmermann2025chemical} introduced FART, a chemical language model fine-tuned on 15K molecular tastants for multi-class taste prediction across sweet, bitter, sour, and umami categories, achieving over 91\% accuracy. Prior to FART, molecular taste prediction relied on binary classifiers specialized to individual taste classes, using representations such as Morgan fingerprints or molecular descriptors paired with tree-based models. \citet{thomas2025what} evaluated several LLMs on sensory prediction tasks constructed from an earlier version of the NECTAR dataset, finding that LLMs can exceed human food scientist performance on experimental design for plant-based meats as judged by expert evaluators. However, that work considered only text-based inputs and did not attempt to improve upon foundation model performance with learned representations. \textsc{TasteBench} extends both lines of work: it uses FART's molecular representations as one feature modality while introducing a food-level ranking task that connects molecular-level predictions to product-level sensory outcomes, a link that no prior work has established.

\paragraph{Olfactory Prediction.} Computational olfaction provides the closest methodological precedent for our work.~\citet{lee2023principal} trained a graph neural network to construct a principal odor map (POM) that predicted odor qualities of novel odorants from molecular structure; on a prospective validation set of 400 molecules, the model matched the panel mean more closely than the median human panelist, establishing the benchmarking paradigm we adopt for our human performance baseline (Section~\ref{sec:human-baseline}). ~\citet{tom2025molecules} extended POM to molecular mixtures with POMMix, using attention mechanisms to aggregate single-molecule GNN embeddings into mixture representations. This molecule-to-mixture aggregation problem is structurally analogous to our compound-to-food aggregation pipeline (Section~\ref{sec:food-level-features}), though we operate in the taste rather than olfactory domain and aggregate across ingredients rather than mixture components. ~\citet{feng2026smellnet} introduced SmellNet, a large-scale sensor-based benchmark for real-world odor classification with 180K time steps across 50 substances, providing a benchmarking template for chemosensory perception that complements our label-based approach.

\section{Theoretical Analysis}\label{sec:theoretical}
A natural question is: how good must the sensory prediction model be in order to be useful? In a simplified model, we can use properties of the truncated bivariate normal to begin to understand the relationship between performance of the predictive model and downstream product improvements. Let $S$ be the property of interest, e.g. sensory similarity to the target animal product, and let $P$ be the predicted value, with $S \sim \mathcal{N}(\mu_S, \sigma^2_{S})$, and $P \sim \mathcal{N}(\mu_P, \sigma^2_{P})$.

The quantity we want to understand is $\frac{E[S | P > T] - E[S]}{E[S]} = \frac{E[S | P > T] - \mu_S}{\mu_S}$: the percent improvement in the target attribute when using our predictive model to select a subset of candidate formulations.

\begin{proposition}\label{main-prep}
$\frac{E[S | P > T] - \mu_S}{\mu_S} = \frac{\sigma_{S}\rho}{\mu_S} \frac{\phi(\frac{T - \mu_{P}}{\sigma_{P}})}{1 - \Phi(\frac{T - \mu_{P}}{\sigma_{P}})}$, where $\Phi(x)$ is the standard normal CDF, $\phi(x)$ is the standard normal PDF, and $\rho$ is the correlation between $S$ and $P$.
\end{proposition}
\begin{proof}
The proposition follows from standardizing $S$ and $P$ and applying a result from~\cite{maddala1983limited} that $E[X | Y > T] = \rho \frac{\phi(T)}{1 - \Phi(T)}$ for standard bivariate normals $X$ and $Y$ with correlation $\rho$.
\end{proof}



For example, in the burgers category the $\frac{\sigma_S}{\mu_S}$ ratio was 0.19, though this is the variance of products on the market rather than an internal candidate pool. With a correlation between the predictor and the true property of 0.81 (both values from the NNLS ensemble; see Table~\ref{tab:per-category-nnls} for per-category inputs), and taking the top 5\% of candidates, this translates to a percent improvement of 32\% in the target property when using the predictive model. This  increases linearly as the variance of the pool of candidates and the correlation between predictive model and true property increase. This number is illustrative; the realized improvement on a real internal pool depends on its variance, which we cannot observe.

\section{Experimental Details}\label{app:exp-details}

\subsection{Cost Hierarchy}\label{app:cost-hierarchy}

Figure~\ref{fig:cost-hierarchy} visualizes the hierarchy of cost in sustainable protein evaluation. We study whether lower-cost features can predict the gold standard evaluation of a large-scale, representative sensory panel. By benchmarking lower-cost computational proxies (e.g., zero-shot LLMs) against the highest-cost human ground truth, \textsc{TasteBench} assesses the viability of these novel evaluation methodologies.

\begin{figure}[h!]
\centering
\includegraphics[width=\textwidth]{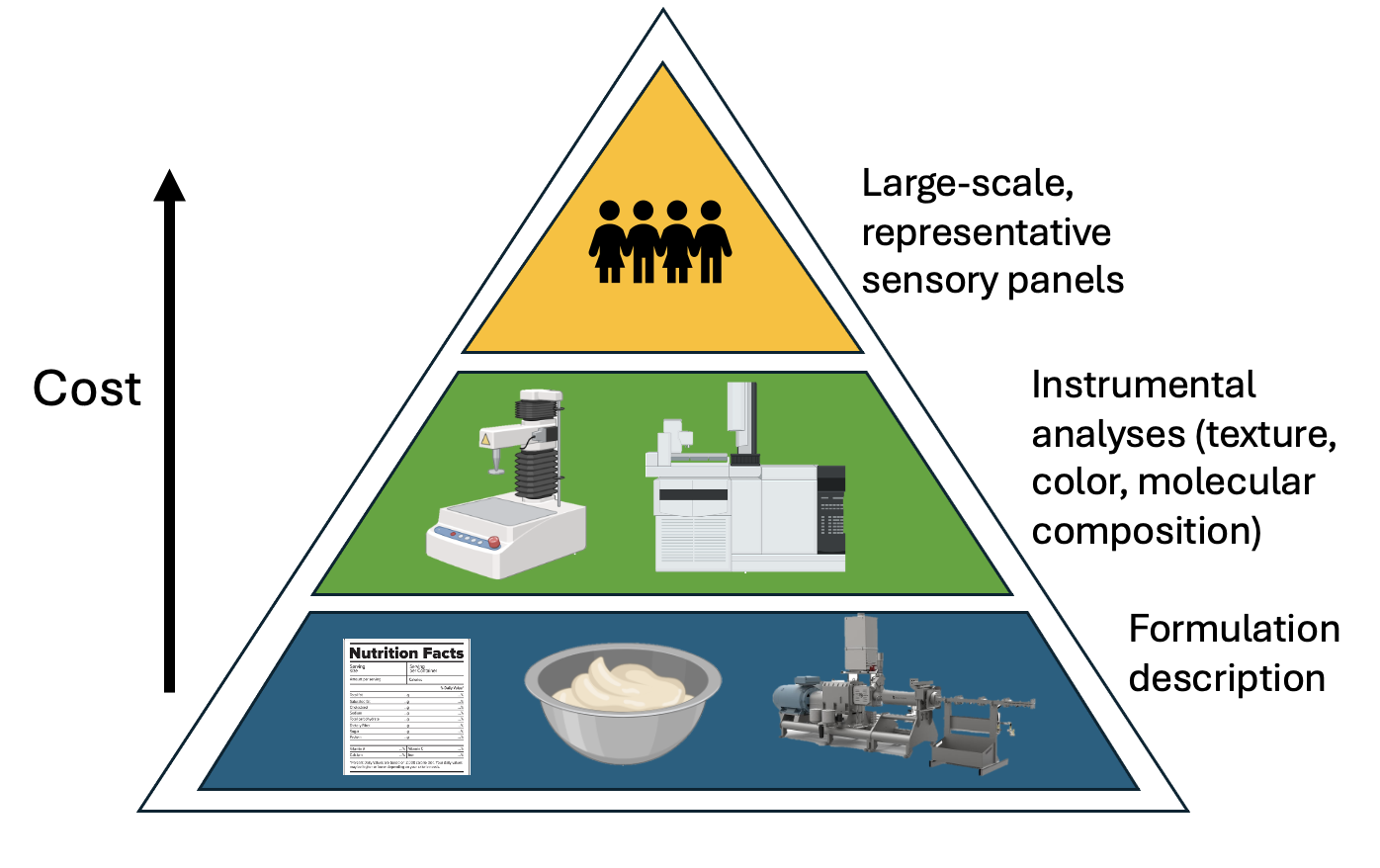}
\caption{Hierarchy of cost in sustainable protein evaluation.}
\label{fig:cost-hierarchy}
\end{figure}

\subsection{Privacy-Preserving Kaggle Competition}\label{app:privacy-details}
The most harmful attack is fine-grained pairwise inference: an attacker submits rankings that swap two specific products and reads the panel's ordering from the resulting leaderboard delta of $\pm\frac{1}{935} \sim 0.001$. Two layers of defense raise the cost of this attack. First, leaderboard scores are rounded to two decimal places, an order of magnitude coarser than the single-pair signal, and participants are restricted to five submissions per day. 
Second, Kaggle's public/private leaderboard split places 11\% of pairs in a private holdout that is not probeable during the competition and is frozen afterward, so an attacker cannot adaptively target a specific brand's product.  

\subsection{Dataset Licenses and Terms of Use}\label{app:dataset-licenses}
\begin{itemize}[itemsep=2pt, topsep=2pt, parsep=0pt]
    \item \textbf{NECTAR sensory ratings, ingredient/nutrition CSVs,
      and product images.} Licensed under Creative Commons
      Attribution-NonCommercial 4.0 International (CC BY-NC 4.0) by
      NECTAR, with access gated under NECTAR's NDA. Researchers can request access via the Google Form linked in Section~\ref{sec:datasets}. 

    \item \textbf{Taste Like CPG product directory.} This dataset has been acquired by NECTAR
    and can be requested via the \href{https://docs.google.com/forms/d/e/1FAIpQLSdYXkB1yEOt_gMBrnHIpHv_xX9peh5NS1kdbfuY8-6-77kSug/viewform?usp=preview}{Google Form} linked in Section~\ref{sec:datasets}.

    \item \textbf{FoodAtlas v$4.0$.} Licensed under the Apache License
      2.0. 


    \item \textbf{FartDB (FART train/val/test splits).} Licensed
      under the MIT License. Pinned to commit
      \texttt{bde90e6562ce5d248e76af791fab29ffc9ae901b} of
      \url{https://github.com/fart-lab/fart}; vendored at
      \texttt{molecular/data/splits/}. The upstream project requests
      citation of \citet{zimmermann2025chemical} for any work using
      the dataset.
\end{itemize}

\subsection{LLM prompt}\label{app:llm-prompts}

\begin{figure}[H]
\begin{tcolorbox}[colback=lightyellow,
  colframe=yellow, coltitle=black,
  title=System prompt]
\scriptsize
\begin{verbatim}
<role>
You are an expert food scientist.
Your task is to predict which of two
plant-based products would be perceived
as more similar to its animal-based
counterpart in a blind taste test.
</role>

<output_format>
Respond with ONLY the number 1 or 2.
No other text.
</output_format>
\end{verbatim}
\end{tcolorbox}

\vspace{2pt}

\begin{tcolorbox}[colback=lightyellow,
  colframe=yellow, coltitle=black,
  title=User prompt (all modalities)]
\scriptsize
\begin{verbatim}
<context>
Two plant-based {category} products
are compared for similarity to their
animal-based {category} counterpart.
Consider ingredients, nutrition facts,
and visual appearance.
</context>

<data>
<product id="1">
<ingredients order="descending_by_weight">
{ingredient list 1}
</ingredients>
<nutrition unit="per_100g">
{nutrition facts 1}
</nutrition>
</product>

<product id="2">
(same structure)
</product>
</data>

<images>
The first image shows Product 1 and
the second image shows Product 2.
</images>

<task>
In a blind taste test, which product
would omnivores rank as more similar
to the animal-based {category}?
</task>
\end{verbatim}
\end{tcolorbox}
\caption{Prompt template for LLM-based
models. Feature ablation removes the
corresponding XML blocks.}
\label{fig:llm-prompt}
\end{figure}

\subsection{Nutrition feature selection}\label{app:nutrition-selection}

We include only nutrition label features where nutritional distance is a valid proxy for sensory distance, selecting subsets based on the dominant sensory drivers in each product type
(Table~\ref{tab:nutrition-features}).

\begin{table}[H]
\centering
\small
\caption{Nutrition features by category group. All values
normalized to per-100g serving.}
\label{tab:nutrition-features}
\begin{tabular}{@{}ll@{}}
\toprule
\textbf{Group} & \textbf{Features} \\
\midrule
Meat (14 cat.) & Total Fat, Sodium, Protein, Dietary Fiber \\
Non-Sweet Dairy (6 cat.) & Total Fat, Sodium \\
Cheese (2 cat.) & Total Fat, Sodium, Total Carbs \\
Sweet Dairy (2 cat.) & Total Fat, Total Sugars \\
\bottomrule
\end{tabular}
\end{table}

Feature selections are supported by: total fat, sodium, protein, and dietary fiber
in meat
analogs~\citep{thomas2025what,godschalk2022meat,bakhsh2021methylcellulose};
fat and sodium in non-sweet
dairy~\citep{magwere2025comparative,liem2011sodium}; total fat, sodium, and total carbohydrate
in
cheese~\citep{rune2026plantcheese,ganesan2014sodium,grasso2021plantcheese};
and total fat and total sugars in sweet
dairy~\citep{guinard1997sugars,dandrea2023yogurt}.

\subsection{Category Subsets}
\label{app:category-subsets}

Per-category sample sizes in NECTAR are too small for category-specific
modeling (as few as 5 products per category, mean $\approx 9$). We
therefore group the 24 categories into the same four subsets used for
nutrition feature selection (Appendix~\ref{app:nutrition-selection}):

\begin{description}[leftmargin=1.2em,itemsep=2pt]
\item[\textit{Meat} ($n{=}121$):] bacon, bratwurst, breakfast sausages,
burgers, breaded chicken filet, chicken strips, deli ham, deli turkey,
hot dogs, meatballs, nuggets, pulled pork, steak, unbreaded chicken
breast.
\item[\textit{Nonsweet dairy} ($n{=}62$):] barista milk, butter, cream
cheese, creamer, milk, sour cream.
\item[\textit{Cheese} ($n{=}16$):] cheddar cheese, mozzarella.
\item[\textit{Sweet dairy} ($n{=}16$):] ice cream (hard serve), yogurt.
\end{description}

Each product is encoded as one dimension of a 4-dimensional one-hot
indicator over these subsets for the supervised models. Total:
$215$ plant-based products.

\subsection{Food-Level Model Training Details}
\label{app:food-level-training}

\paragraph{Featurization.} See Section~\ref{sec:food-level-features} for the
five feature blocks (\textbf{S}~category, \textbf{N}~nutrition, \textbf{C}~compound,
\textbf{T}~ingredient text, \textbf{I}~image), the FDA-rank top-3 ingredient
aggregation, and the log-concentration weighting motivated by the
Weber--Fechner law (ablation in Table~\ref{tab:ingredient-agg-ablation}).

\paragraph{Preprocessing.} Per-modality \texttt{StandardScaler} fit on the
training fold; high-dimensional modalities (C, T, I) additionally
PCA-reduced to $95\%$ explained variance, fit within each LOOCV fold to
prevent train-time leakage. Missing image embeddings are imputed by the
$k{=}5$ nearest neighbors within the same category over the available
non-image features; missing nutrition values are zero-imputed.

\paragraph{LOOCV protocol.} $215$ folds, one plant-based product held out per
fold. For pairwise models (Bradley--Terry, hierarchical BT, Kernel RankSVM),
training pairs are constructed from within-category preference comparisons
of training-fold products only; the held-out product never appears on
either side of any training pair. For pointwise models (Ridge, LightGBM),
targets are mean panelist similarity ratings. Predictions across folds are
concatenated into a single per-product OOF prediction file at
\texttt{food\_similarity/results/oof\_predictions/\{model\}\_S\{subset\}.csv},
which is the input to all paper render scripts.

\paragraph{Hyperparameters.} All models use random seed $42$ where
applicable. Defaults below are committed in
\texttt{food\_similarity/models/} and were not swept; this is intentional
to keep the model comparison apples-to-apples and avoid the appearance of
per-model tuning advantage.
\begin{itemize}
    \item \textbf{Ridge.} \texttt{sklearn.linear\_model.Ridge}, $\alpha=1.0$,
      \texttt{fit\_intercept=True}.
    \item \textbf{Bradley--Terry.} \texttt{sklearn.linear\_model.LogisticRegression}
      on within-category feature-pair differences. $C=1.0$,
      \texttt{max\_iter}$=1000$.
    \item \textbf{Hierarchical Bradley--Terry.} Per-subset BT for
      \emph{meat}, \emph{nonsweet dairy}, and \emph{other dairy} (the latter
      merging cheese and sweet dairy, which have $2$ categories each --- too
      few for stable group-level estimation); empirical Bayes shrinkage
      toward the global coefficients. Shrinkage strength
      $\lambda_s = \alpha / (\alpha + n_{\text{pairs},s})$ with $\alpha$ set
      to the median of $\{n_{\text{pairs},s}\}$, so a median-sized subset
      receives $50\%$ shrinkage.
    \item \textbf{Kernel RankSVM.} \texttt{sklearn.svm.SVC} on pairwise
      difference features with RBF kernel, $C=1.0$,
      $\gamma=$ \texttt{"scale"}, \texttt{max\_iter}$=5000$.
    \item \textbf{LightGBM.} Regression objective (RMSE),
      \texttt{n\_estimators}$=100$, \texttt{learning\_rate}$=0.1$,
      \texttt{num\_leaves}$=31$, \texttt{min\_child\_samples}$=5$.
      \texttt{deterministic=True}, \texttt{force\_row\_wise=True},
      \texttt{num\_threads=1} for byte-identical reproduction.
\end{itemize}

\paragraph{Distance baselines (MMRF).} For each category, compute the mean
feature vector of animal-based reference products per modality. Score each
plant-based product by its distance to that centroid (cosine or
Euclidean), rank-normalize within category per modality, and average ranks
across enabled modalities. No training, no fold structure.

\paragraph{Zero-shot LLMs.} Prompt template, model versions (Gemini 3.1
Pro, Qwen 3.5 397B-A17B), and modality combinations are documented in
Appendix~\ref{app:llm-prompts}. Pair ordering is deterministic via MD5
hash to neutralize position bias; pairwise judgments are aggregated to a
per-product score by win rate (number of pairs won divided by number of
pairs participated in --- equivalent to Copeland's method), and within-category
ranks are computed from these scores.

\paragraph{Ensembles.} Nested LOOCV combines Bradley--Terry (over Category Subset, Nutrition, Compound, Ingredient Text, Image features, abbreviated SNCTI)
with Gemini 3.1 Pro. For each outer-fold held-out product $i$, the
inner-loop BT predictions for the evaluation partners $j$ are produced from
a BT trained on $\{k : k \neq i,\ k \neq j\}$ ($213$ products), preventing
partner-level label leakage. We report three combination strategies fit
per fold: non-negative least squares (\texttt{scipy.optimize.nnls}),
arithmetic mean, and rank averaging.

\paragraph{Confidence intervals.} BCa bootstrap~\citep{efron1987better},
$n_{\text{boot}}=10{,}000$, stratified by category, seed $42$. CIs are
computed once per (model, metric, scope) and cached at
\texttt{food\_similarity/results/cis\_*.csv} so table re-rendering does
not re-run the bootstrap.

\subsection{GNN Training Details}
\label{app:gnn-training}

\paragraph{Featurization.} Each molecule's canonicalized SMILES is parsed with RDKit and featurized using Chemprop's default \texttt{SimpleMoleculeMolGraphFeaturizer}. Atom features (72 dimensions): atomic number, total degree, formal charge, chirality tag, hydrogen count, hybridization, aromaticity, atomic mass. Bond features (14 dimensions): bond type, conjugation, ring membership, stereochemistry.

\paragraph{Architecture.} Directed bond-level message passing (\texttt{BondMessagePassing}) with hidden dimension 300 and batch normalization, mean aggregation over atoms, and a single-hidden-layer classification feed-forward network (also hidden dimension 300).

\paragraph{Training.} Adam optimizer with a Noam-like learning rate schedule (initial $10^{-4}$, peak $10^{-3}$, final $10^{-4}$). Batch size 50, maximum 50 epochs, cross-entropy loss (optionally class-weighted, see hyperparameter grid below), early stopping on validation loss with patience 10. Within-run checkpoint selection: minimum validation loss. Random seed: 42.

\paragraph{Hyperparameter grid.} The 12 configurations are the Cartesian product of depth $\in \{3, 5\}$, dropout $\in \{0, 0.2\}$, and class weighting $\in \{\text{none}, \text{inverse-frequency}, \text{sqrt-inverse-frequency}\}$. Across-grid selection: maximum validation macro-F1, committed before evaluating on the FART test set or any NECTAR data. Validation macro-F1 for all 12 configurations is reported in Table~\ref{tab:gnn-grid}; the selected configuration leads the runner-up by 0.8\,pp.

\begin{center}
\captionof{table}{Validation macro-F1 for all 12 configurations in the GNN hyperparameter grid (depth $\times$ dropout $\times$ class weighting). The selected configuration (\textbf{bold}) is the highest val macro-F1; this checkpoint is used in Tables~\ref{tab:molecular-prediction} and \ref{tab:encoder-transfer}. Selection committed before evaluating on the FART test set or any NECTAR data.}
\label{tab:gnn-grid}
\footnotesize
\setlength{\tabcolsep}{4pt}
\renewcommand{\arraystretch}{0.95}
\begin{tabular}{@{}cccc@{}}
\toprule
Depth & Dropout & Class weighting & Val macro-F1 \\
\midrule
\textbf{5} & \textbf{0.2} & \textbf{none} & \textbf{.825} \\
5 & 0 & none & .817 \\
5 & 0.2 & sqrt-inverse-frequency & .804 \\
3 & 0 & none & .800 \\
3 & 0.2 & none & .789 \\
3 & 0 & sqrt-inverse-frequency & .776 \\
5 & 0 & sqrt-inverse-frequency & .771 \\
3 & 0.2 & sqrt-inverse-frequency & .764 \\
3 & 0 & inverse-frequency & .759 \\
5 & 0.2 & inverse-frequency & .713 \\
5 & 0 & inverse-frequency & .711 \\
3 & 0.2 & inverse-frequency & .696 \\
\bottomrule
\end{tabular}
\end{center}

\paragraph{Embedding extraction.} The 300-dimensional penultimate-layer representation used as the compound encoder in Table~\ref{tab:encoder-transfer} is the activation after the classifier FFN's input projection, before the final classification head.

\subsection{Compute Resources}
\label{app:compute}

All supervised models, distance-based baselines, and the D-MPNN graph neural network were run on a single MacBook Pro (Apple M4 Pro, 14 cores, 48\,GB RAM) running macOS, using CPU only. No GPU acceleration was used. Bootstrap confidence intervals (10{,}000 resamples per evaluation) were computed in parallel across cores. End-to-end reproducibility takes under a few hours on this workstation, plus the OpenRouter API calls for LLM evaluation.

LLM evaluations (Gemini 3.1 Pro and Qwen 3.5 397B-A17B) were performed via the OpenRouter API using the prompt template in Appendix~\ref{app:llm-prompts}.

\section{Reproducibility}
\label{app:reproducibility}

This section outlines the files we release to reproduce the results in the paper. The high-level access structure is summarized in Section~\ref{sec:access-structure}; this appendix lists each artifact, its location, and what it backs in the paper. All file paths below are
relative to the repository at
\url{https://github.com/Food-Intelligence-Lab/tastebench/}.

\paragraph{Food-level task --- supervised baselines.}
The supervised baselines (Ridge, Bradley--Terry, hierarchical BT,
Kernel RankSVM, LightGBM) are fit per fold under leave-one-out
cross-validation, producing $215$ per-fold fits per model. Rather than serialize these intermediate fits, which would inflate the artifact
without adding reproducibility value, we ship the deterministic
per-product out-of-fold predictions at
\texttt{food\_similarity/results/oof\_predictions/}, one CSV per
(model, feature-subset) combination. Each CSV is the canonical input
to its corresponding paper render script. Hyperparameters, training
protocol, and the random seed ($42$) are documented in
Appendix~\ref{app:food-level-training}; given the released code,
cached feature matrices, and fixed seed, any reviewer can regenerate
the OOFs byte-identically by running
\texttt{food\_similarity/reproduce.sh}.

\paragraph{Food-level task --- LLM baselines.}
Gemini~3.1~Pro and Qwen 3.5 397B-A17B are accessed via the OpenRouter
API; we do not redistribute their weights. The deterministic outputs
(prompt completions and the resulting per-product win-rate scores)
ship at
\texttt{food\_similarity/results/oof\_predictions/llm\_*.csv}. Each is
a deterministic function of the prompt template
(Appendix~\ref{app:llm-prompts}), the MD5-hashed pair-ordering, and
the provider's serving snapshot at the time of the run.

\paragraph{Molecular task --- GNN checkpoints.}
All twelve trained D-MPNN checkpoints from the hyperparameter grid
(Appendix~\ref{app:gnn-training}) ship at
\texttt{molecular/results/grid/run\_*/ckpt.pt}, each accompanied by a
\texttt{config.yaml} (the hyperparameters that produced it) and a
\texttt{val\_metrics.json} (per-class
validation F1; the aggregated table across all runs is at
\texttt{molecular/results/grid/grid\_summary.csv}). The
validation-best run is symlinked at
\texttt{molecular/results/grid/best/}, and is the checkpoint whose
accuracy is reported in the GNN row of
Table~\ref{tab:molecular-prediction} and whose penultimate-layer
activations form the compound encoder evaluated in
Table~\ref{tab:encoder-transfer}. Total checkpoint footprint is
approximately $15$\,MB. Held-out FART test predictions for each run
are committed at
\texttt{molecular/results/grid/run\_*/fart\_test\_eval/predictions.parquet}.

\paragraph{Pretrained external encoders.}
Three external models are used for feature extraction and are not
redistributed:
\begin{itemize}
    \item \textbf{DINOv3~ViT-L/16}~\citep{oquab2024dinov}, fetched
      from the official Hugging Face release (food-level image
      encoder).
    \item \textbf{Qwen3-Embedding-0.6B}, fetched from the official
      Hugging Face release (food-level ingredient-text encoder).
    \item \textbf{FART}~\citep{zimmermann2025chemical}, fetched from
      the authors' GitHub release (compound encoder used by both
      tasks).
\end{itemize}
Cached embeddings (one matrix per modality, keyed by product code or
canonical SMILES) ship at \texttt{shared/data/caches/} and are
byte-checksummed against
\texttt{food\_similarity/results/input\_cache\_sha256.txt}.
Regeneration scripts that produce these caches from the upstream
model releases live at \texttt{shared/scripts/prepare\_caches/}.

\paragraph{Confidence-interval caches.}
$95\%$ BCa bootstrap intervals~\citep{efron1987better}
($n_{\text{boot}}=10{,}000$, stratified by category, seed $42$) are
cached alongside each table at
\texttt{food\_similarity/results/cis\_*.csv} and
\texttt{molecular/results/cis\_*.csv}. The cached intervals make
table re-rendering fast (no bootstrap re-run); the seeds and
$n_{\text{boot}}$ are baked into both the cache filenames and each
render script's header comment so any drift is visible in a diff.

\paragraph{End-to-end reproduction scripts.}
Two top-level scripts cover the two reviewer paths.
\texttt{verify\_paper.sh} re-renders every paper table and figure
from the committed out-of-fold predictions and parquet predictions,
with no external data and no model training; this is the lightweight
reviewer verification path. \texttt{food\_similarity/reproduce.sh}
regenerates the food-level OOFs from the gated NECTAR ratings
(requires NECTAR access; see \texttt{data/GATED.md}); this is the
heavyweight retrain path. The molecular GNN grid is regenerated via
\texttt{molecular/scripts/submit\_grid.sh dmpnn\_grid} followed by
\texttt{python -m molecular.src.train.select\_best\_and\_evaluate};
the D-MPNN trains on CPU and no GPU is required.

\section{Additional Results}

\subsection{Model Ablation}\label{app:add-results}

We report two complementary ablations. Table~\ref{tab:ablation} sweeps the supervised baselines and unsupervised distance models (MMRF) over all $15$ non-empty subsets of the $\{\text{N},\text{C},\text{T},\text{I}\}$ feature blocks defined in Section~\ref{sec:food-level-features}; Table~\ref{tab:ablation-llm} sweeps Gemini~3.1~Pro and Qwen~3.5~397B-A17B across all $7$ non-empty subsets of $\{\text{ingredients}, \text{nutrition}, \text{image}\}$. Pairwise accuracy with $95\%$ BCa CIs is reported in both. The headline result - the SNCTI / ingredients+image configuration - is the one carried into the main table; the ablation documents that no smaller subset matches it within its CI for either the supervised or LLM family.

\begin{center}
\captionof{table}{Feature ablation: pairwise accuracy (point on top, 95\% BCa CI from 10{,}000 resamples below) for supervised models and unsupervised distance predictors. S (category subset) is always included for supervised models. N = nutrition, C = compound, T = text, I = image. BT = Bradley--Terry, HBT = Hierarchical BT, KSVM = Kernel RankSVM, LGBM = LightGBM. \textbf{Bold} = best subset per model; $^\dag$ = CI overlaps the column leader (no significant difference at 95\%). Values below .500 indicate worse-than-random. Differences between subsets are within noise: for Bradley--Terry, nutrition$+$text scores .641 [.602,.700] and adding the remaining modalities lowers the point estimate (compound .618, image .606, both .610), but every one of these lies inside the nutrition$+$text interval.}
\label{tab:ablation}
\footnotesize
\setlength{\tabcolsep}{3pt}
\renewcommand{\arraystretch}{1.05}
\begin{tabular}{@{}lccccccc@{}}
\toprule
 & \multicolumn{5}{c}{Supervised (S+subset)} & \multicolumn{2}{c}{MMRF} \\
\cmidrule(lr){2-6} \cmidrule(lr){7-8}
Subset & Ridge & BT & HBT & KSVM & LGBM & Cos & L2 \\
\midrule
N & \shortstack{.447 \\ {\scriptsize [.389,.497]}} & \shortstack{.498 \\ {\scriptsize [.440,.553]}} & \shortstack{.493 \\ {\scriptsize [.439,.542]}} & \shortstack{.509 \\ {\scriptsize [.459,.564]}} & \shortstack{.511$^\dag$ \\ {\scriptsize [.455,.567]}} & \shortstack{.566$^\dag$ \\ {\scriptsize [.517,.624]}} & \shortstack{.549$^\dag$ \\ {\scriptsize [.496,.609]}} \\
C & \shortstack{.547$^\dag$ \\ {\scriptsize [.496,.601]}} & \shortstack{.572$^\dag$ \\ {\scriptsize [.521,.628]}} & \shortstack{.565$^\dag$ \\ {\scriptsize [.514,.622]}} & \shortstack{.563$^\dag$ \\ {\scriptsize [.513,.617]}} & \shortstack{.548$^\dag$ \\ {\scriptsize [.498,.603]}} & \shortstack{.498$^\dag$ \\ {\scriptsize [.439,.556]}} & \shortstack{.504$^\dag$ \\ {\scriptsize [.444,.562]}} \\
T & \shortstack{.594$^\dag$ \\ {\scriptsize [.549,.655]}} & \shortstack{.614$^\dag$ \\ {\scriptsize [.575,.671]}} & \shortstack{.567$^\dag$ \\ {\scriptsize [.518,.626]}} & \shortstack{.597$^\dag$ \\ {\scriptsize [.555,.654]}} & \shortstack{.546$^\dag$ \\ {\scriptsize [.495,.605]}} & \shortstack{.499$^\dag$ \\ {\scriptsize [.445,.551]}} & \shortstack{.520$^\dag$ \\ {\scriptsize [.468,.572]}} \\
I & \shortstack{.411 \\ {\scriptsize [.352,.453]}} & \shortstack{.514 \\ {\scriptsize [.459,.568]}} & \shortstack{.540$^\dag$ \\ {\scriptsize [.486,.595]}} & \shortstack{.512 \\ {\scriptsize [.456,.565]}} & \shortstack{.473 \\ {\scriptsize [.416,.524]}} & \shortstack{.559$^\dag$ \\ {\scriptsize [.504,.625]}} & \shortstack{.569$^\dag$ \\ {\scriptsize [.515,.635]}} \\
\midrule
NC & \shortstack{.548$^\dag$ \\ {\scriptsize [.495,.604]}} & \shortstack{.570$^\dag$ \\ {\scriptsize [.519,.629]}} & \shortstack{.563$^\dag$ \\ {\scriptsize [.514,.615]}} & \shortstack{.559$^\dag$ \\ {\scriptsize [.509,.615]}} & \shortstack{.543$^\dag$ \\ {\scriptsize [.489,.600]}} & \shortstack{.545$^\dag$ \\ {\scriptsize [.496,.599]}} & \shortstack{.541$^\dag$ \\ {\scriptsize [.489,.595]}} \\
NT & \shortstack{.589$^\dag$ \\ {\scriptsize [.542,.650]}} & \shortstack{\textbf{.641} \\ {\scriptsize [.603,.706]}} & \shortstack{.548$^\dag$ \\ {\scriptsize [.495,.605]}} & \shortstack{.595$^\dag$ \\ {\scriptsize [.552,.651]}} & \shortstack{.560$^\dag$ \\ {\scriptsize [.510,.616]}} & \shortstack{.558$^\dag$ \\ {\scriptsize [.512,.612]}} & \shortstack{.561$^\dag$ \\ {\scriptsize [.515,.619]}} \\
NI & \shortstack{.440 \\ {\scriptsize [.384,.488]}} & \shortstack{.529 \\ {\scriptsize [.477,.583]}} & \shortstack{.539$^\dag$ \\ {\scriptsize [.487,.596]}} & \shortstack{.510 \\ {\scriptsize [.453,.563]}} & \shortstack{.511$^\dag$ \\ {\scriptsize [.455,.565]}} & \shortstack{\textbf{.579} \\ {\scriptsize [.534,.637]}} & \shortstack{.581$^\dag$ \\ {\scriptsize [.536,.641]}} \\
CT & \shortstack{.589$^\dag$ \\ {\scriptsize [.544,.645]}} & \shortstack{.611$^\dag$ \\ {\scriptsize [.573,.668]}} & \shortstack{.539$^\dag$ \\ {\scriptsize [.485,.596]}} & \shortstack{\textbf{.619} \\ {\scriptsize [.581,.674]}} & \shortstack{.557$^\dag$ \\ {\scriptsize [.503,.616]}} & \shortstack{.487 \\ {\scriptsize [.435,.532]}} & \shortstack{.497$^\dag$ \\ {\scriptsize [.444,.548]}} \\
CI & \shortstack{.471 \\ {\scriptsize [.411,.522]}} & \shortstack{.537 \\ {\scriptsize [.482,.592]}} & \shortstack{.580$^\dag$ \\ {\scriptsize [.536,.635]}} & \shortstack{.576$^\dag$ \\ {\scriptsize [.526,.635]}} & \shortstack{.543$^\dag$ \\ {\scriptsize [.488,.602]}} & \shortstack{.550$^\dag$ \\ {\scriptsize [.498,.609]}} & \shortstack{.555$^\dag$ \\ {\scriptsize [.503,.615]}} \\
TI & \shortstack{.595$^\dag$ \\ {\scriptsize [.551,.653]}} & \shortstack{.608$^\dag$ \\ {\scriptsize [.569,.661]}} & \shortstack{.623$^\dag$ \\ {\scriptsize [.582,.680]}} & \shortstack{.613$^\dag$ \\ {\scriptsize [.576,.669]}} & \shortstack{.500$^\dag$ \\ {\scriptsize [.446,.551]}} & \shortstack{.544$^\dag$ \\ {\scriptsize [.493,.599]}} & \shortstack{.577$^\dag$ \\ {\scriptsize [.530,.634]}} \\
\midrule
NCT & \shortstack{\textbf{.606} \\ {\scriptsize [.567,.661]}} & \shortstack{.618$^\dag$ \\ {\scriptsize [.581,.678]}} & \shortstack{.534$^\dag$ \\ {\scriptsize [.479,.593]}} & \shortstack{.617$^\dag$ \\ {\scriptsize [.579,.673]}} & \shortstack{\textbf{.587} \\ {\scriptsize [.538,.647]}} & \shortstack{.530$^\dag$ \\ {\scriptsize [.484,.579]}} & \shortstack{.550$^\dag$ \\ {\scriptsize [.499,.606]}} \\
NCI & \shortstack{.504 \\ {\scriptsize [.449,.560]}} & \shortstack{.559$^\dag$ \\ {\scriptsize [.506,.618]}} & \shortstack{.591$^\dag$ \\ {\scriptsize [.548,.645]}} & \shortstack{.573$^\dag$ \\ {\scriptsize [.522,.630]}} & \shortstack{.521$^\dag$ \\ {\scriptsize [.468,.576]}} & \shortstack{.574$^\dag$ \\ {\scriptsize [.527,.632]}} & \shortstack{.571$^\dag$ \\ {\scriptsize [.522,.628]}} \\
NTI & \shortstack{.586$^\dag$ \\ {\scriptsize [.539,.644]}} & \shortstack{.606$^\dag$ \\ {\scriptsize [.567,.659]}} & \shortstack{\textbf{.630} \\ {\scriptsize [.590,.690]}} & \shortstack{.612$^\dag$ \\ {\scriptsize [.574,.667]}} & \shortstack{.539$^\dag$ \\ {\scriptsize [.484,.593]}} & \shortstack{.568$^\dag$ \\ {\scriptsize [.521,.627]}} & \shortstack{.587$^\dag$ \\ {\scriptsize [.539,.644]}} \\
CTI & \shortstack{.575$^\dag$ \\ {\scriptsize [.531,.630]}} & \shortstack{.621$^\dag$ \\ {\scriptsize [.584,.674]}} & \shortstack{.622$^\dag$ \\ {\scriptsize [.579,.679]}} & \shortstack{\textbf{.619} \\ {\scriptsize [.579,.675]}} & \shortstack{.539$^\dag$ \\ {\scriptsize [.485,.597]}} & \shortstack{.530$^\dag$ \\ {\scriptsize [.478,.586]}} & \shortstack{.551$^\dag$ \\ {\scriptsize [.499,.609]}} \\
\midrule
NCTI & \shortstack{.579$^\dag$ \\ {\scriptsize [.534,.633]}} & \shortstack{.610$^\dag$ \\ {\scriptsize [.571,.662]}} & \shortstack{.613$^\dag$ \\ {\scriptsize [.570,.671]}} & \shortstack{.618$^\dag$ \\ {\scriptsize [.578,.675]}} & \shortstack{.558$^\dag$ \\ {\scriptsize [.507,.615]}} & \shortstack{.566$^\dag$ \\ {\scriptsize [.519,.624]}} & \shortstack{\textbf{.591} \\ {\scriptsize [.546,.649]}} \\
\bottomrule
\end{tabular}
\end{center}

\begin{center}
\captionof{table}{Input ablation: pairwise accuracy with 95\% BCa CIs (10{,}000 resamples) for zero-shot LLMs across modality subsets. Ingr.\ = ingredient list text; Nutr.\ = nutrition facts; Img.\ = product image. \textbf{Bold} = best per model; $^\dag$ = CI overlaps the column leader (no significant difference at 95\%).}
\label{tab:ablation-llm}
\footnotesize
\setlength{\tabcolsep}{4pt}
\renewcommand{\arraystretch}{0.95}
\begin{tabular}{@{}lcc@{}}
\toprule
Input & Gemini & Qwen \\
\midrule
Ingr. & .607$^\dag$ [.566,.664] & .588$^\dag$ [.540,.648] \\
Nutr. & .592$^\dag$ [.543,.655] & .587$^\dag$ [.537,.646] \\
Img. & .589$^\dag$ [.544,.648] & .572$^\dag$ [.526,.627] \\
Ingr.+Nutr. & .622$^\dag$ [.579,.682] & .604$^\dag$ [.558,.665] \\
\textbf{Ingr.+Img.} & \textbf{.654} [.619,.713] & \textbf{.630} [.591,.688] \\
Nutr.+Img. & .610$^\dag$ [.566,.671] & .573$^\dag$ [.526,.632] \\
All & .625$^\dag$ [.582,.687] & .603$^\dag$ [.556,.664] \\
\bottomrule
\end{tabular}
\end{center}

\subsection{Per-category model comparison}
\label{app:per-cat-comparison}

Table~\ref{tab:per-category} breaks down pairwise accuracy by NECTAR product category for every model in the main results table. The breakdown shows substantial heterogeneity across the $24$ categories, with some categories notably easier than others for all models. Per-category sample sizes are small (as few as $5$ products, median $9$), so per-cell CIs are wide and individual category leaders should not be over-interpreted.

\begin{center}
\captionof{table}{Per-category pairwise accuracy (point on top, 95\% BCa CI from 10{,}000 resamples below; bootstrap restricted to within-category products) for one model per family. \textbf{Bold} = best model per category; $^\dag$ = CI overlaps the row leader (no significant difference at 95\%). $n$ = number of products in category; Pairs $= \binom{n}{2}$ within-category ranking pairs. Wide CIs on small-$n$ categories reflect the small effective sample size and are reported faithfully.}
\label{tab:per-category}
\footnotesize
\setlength{\tabcolsep}{3pt}
\renewcommand{\arraystretch}{1.05}
\begin{tabular}{@{}lrrccc@{}}
\toprule
Category & $n$ & Pairs & Gemini & BT & BT+Gemini \\
\midrule
Bacon & 10 & 45 & \shortstack{\textbf{.556} \\ {\scriptsize [.322,.789]}} & \shortstack{.511$^\dag$ \\ {\scriptsize [.222,.789]}} & \shortstack{.533$^\dag$ \\ {\scriptsize [.289,.778]}} \\
Barista Milk & 10 & 45 & \shortstack{.767$^\dag$ \\ {\scriptsize [.633,.922]}} & \shortstack{.700$^\dag$ \\ {\scriptsize [.521,.889]}} & \shortstack{\textbf{.811} \\ {\scriptsize [.656,.933]}} \\
Bratwurst & 9 & 36 & \shortstack{\textbf{.694} \\ {\scriptsize [.403,.903]}} & \shortstack{.639$^\dag$ \\ {\scriptsize [.375,.875]}} & \shortstack{.667$^\dag$ \\ {\scriptsize [.375,.875]}} \\
Breaded Chicken Filet & 5 & 10 & \shortstack{.300$^\dag$ \\ {\scriptsize [.100,.700]}} & \shortstack{\textbf{.800} \\ {\scriptsize [.300,.950]}} & \shortstack{.300$^\dag$ \\ {\scriptsize [.100,.700]}} \\
Breakfast Sausages & 10 & 45 & \shortstack{.589$^\dag$ \\ {\scriptsize [.311,.844]}} & \shortstack{.578$^\dag$ \\ {\scriptsize [.378,.711]}} & \shortstack{\textbf{.644} \\ {\scriptsize [.333,.900]}} \\
Burgers & 10 & 45 & \shortstack{.756$^\dag$ \\ {\scriptsize [.544,.911]}} & \shortstack{\textbf{.800} \\ {\scriptsize [.622,.922]}} & \shortstack{\textbf{.800} \\ {\scriptsize [.600,.933]}} \\
Butter & 10 & 45 & \shortstack{.689$^\dag$ \\ {\scriptsize [.444,.833]}} & \shortstack{.689$^\dag$ \\ {\scriptsize [.456,.878]}} & \shortstack{\textbf{.711} \\ {\scriptsize [.478,.833]}} \\
Cheddar Cheese & 8 & 28 & \shortstack{\textbf{.768} \\ {\scriptsize [.607,.929]}} & \shortstack{.661$^\dag$ \\ {\scriptsize [.304,.893]}} & \shortstack{.732$^\dag$ \\ {\scriptsize [.589,.911]}} \\
Chicken Strips & 10 & 45 & \shortstack{.522$^\dag$ \\ {\scriptsize [.367,.744]}} & \shortstack{.533$^\dag$ \\ {\scriptsize [.244,.822]}} & \shortstack{\textbf{.578} \\ {\scriptsize [.389,.778]}} \\
Cream Cheese & 10 & 45 & \shortstack{.667$^\dag$ \\ {\scriptsize [.378,.856]}} & \shortstack{.600$^\dag$ \\ {\scriptsize [.322,.811]}} & \shortstack{\textbf{.778} \\ {\scriptsize [.478,.922]}} \\
Creamer & 9 & 36 & \shortstack{\textbf{.875} \\ {\scriptsize [.778,.958]}} & \shortstack{.722$^\dag$ \\ {\scriptsize [.361,.889]}} & \shortstack{.833$^\dag$ \\ {\scriptsize [.625,.944]}} \\
Deli Ham & 9 & 36 & \shortstack{.458$^\dag$ \\ {\scriptsize [.181,.764]}} & \shortstack{\textbf{.708} \\ {\scriptsize [.458,.875]}} & \shortstack{.542$^\dag$ \\ {\scriptsize [.250,.819]}} \\
Deli Turkey & 7 & 21 & \shortstack{.667$^\dag$ \\ {\scriptsize [.333,.857]}} & \shortstack{.714$^\dag$ \\ {\scriptsize [.381,.905]}} & \shortstack{\textbf{.810} \\ {\scriptsize [.571,.952]}} \\
Hot Dogs & 10 & 45 & \shortstack{.667$^\dag$ \\ {\scriptsize [.356,.856]}} & \shortstack{\textbf{.756} \\ {\scriptsize [.589,.922]}} & \shortstack{.733$^\dag$ \\ {\scriptsize [.567,.889]}} \\
Ice Cream Hard Serve & 10 & 45 & \shortstack{.467$^\dag$ \\ {\scriptsize [.256,.756]}} & \shortstack{\textbf{.711} \\ {\scriptsize [.478,.889]}} & \shortstack{.533$^\dag$ \\ {\scriptsize [.333,.789]}} \\
Meatballs & 5 & 10 & \shortstack{\textbf{.600} \\ {\scriptsize [.100,.900]}} & \shortstack{.100$^\dag$ \\ {\scriptsize [.050,.250]}} & \shortstack{\textbf{.600} \\ {\scriptsize [.100,.900]}} \\
Milk & 18 & 153 & \shortstack{.699$^\dag$ \\ {\scriptsize [.536,.833]}} & \shortstack{.693$^\dag$ \\ {\scriptsize [.595,.801]}} & \shortstack{\textbf{.830} \\ {\scriptsize [.716,.928]}} \\
Mozzarella & 8 & 28 & \shortstack{\textbf{.589} \\ {\scriptsize [.268,.911]}} & \shortstack{.214$^\dag$ \\ {\scriptsize [.089,.429]}} & \shortstack{.464$^\dag$ \\ {\scriptsize [.161,.786]}} \\
Nuggets & 10 & 45 & \shortstack{\textbf{.533} \\ {\scriptsize [.322,.756]}} & \shortstack{.356$^\dag$ \\ {\scriptsize [.156,.556]}} & \shortstack{.333$^\dag$ \\ {\scriptsize [.156,.611]}} \\
Pulled Pork & 7 & 21 & \shortstack{\textbf{.667} \\ {\scriptsize [.357,.905]}} & \shortstack{.333$^\dag$ \\ {\scriptsize [.095,.619]}} & \shortstack{.571$^\dag$ \\ {\scriptsize [.333,.833]}} \\
Sour Cream & 5 & 10 & \shortstack{\textbf{.600} \\ {\scriptsize [.250,.900]}} & \shortstack{.500$^\dag$ \\ {\scriptsize [.100,.900]}} & \shortstack{\textbf{.600} \\ {\scriptsize [.250,.900]}} \\
Steak & 9 & 36 & \shortstack{\textbf{.833} \\ {\scriptsize [.722,.931]}} & \shortstack{.556$^\dag$ \\ {\scriptsize [.222,.806]}} & \shortstack{.778$^\dag$ \\ {\scriptsize [.542,.903]}} \\
Unbreaded Chicken Breast & 10 & 45 & \shortstack{.756$^\dag$ \\ {\scriptsize [.522,.911]}} & \shortstack{.600$^\dag$ \\ {\scriptsize [.367,.800]}} & \shortstack{\textbf{.800} \\ {\scriptsize [.600,.911]}} \\
Yogurt & 6 & 15 & \shortstack{\textbf{.467} \\ {\scriptsize [.133,.733]}} & \shortstack{.133$^\dag$ \\ {\scriptsize [.067,.300]}} & \shortstack{.400$^\dag$ \\ {\scriptsize [.133,.700]}} \\
\bottomrule
\end{tabular}
\end{center}

\subsection{Gemini ensemble comparison}
\label{app:per-model-nnls}

Table~\ref{tab:per-model-nnls} shows the pairwise accuracy of combining Gemini 3.1 Pro with each supervised model using a non-negative least squares meta-learner.

\begin{center}
\captionof{table}{Effect of swapping the supervised base in the
BT+Gemini NNLS ensemble. Each row reports the pairwise accuracy of
a supervised model alone (Standalone) and combined with
Gemini~3.1~Pro via nested LOOCV NNLS (+~Gemini NNLS), on the SNCTI
feature set ($n=215$ NECTAR plant-based products, 935 within-category
pairs). Point estimate above, 95\% BCa CI ($10{,}000$ resamples) below.
\textbf{Bold} = best in column; $^\dag$ = CI overlaps the leader
(no significant difference at 95\%).}
\label{tab:per-model-nnls}
\footnotesize
\setlength{\tabcolsep}{6pt}
\renewcommand{\arraystretch}{1.05}
\begin{tabular}{lcc}
\toprule
Supervised model & Standalone & + Gemini NNLS \\
\midrule
Gemini 3.1 Pro (alone) & \shortstack{\textbf{.654} \\ {\scriptsize [.619,.714]}} & — \\
Bradley--Terry & \shortstack{.610$^\dag$ \\ {\scriptsize [.571,.662]}} & \shortstack{\textbf{.683} \\ {\scriptsize [.655,.740]}} \\
Hierarchical BT & \shortstack{.613$^\dag$ \\ {\scriptsize [.571,.671]}} & \shortstack{.665$^\dag$ \\ {\scriptsize [.634,.720]}} \\
Ridge & \shortstack{.579$^\dag$ \\ {\scriptsize [.535,.633]}} & \shortstack{.641$^\dag$ \\ {\scriptsize [.604,.699]}} \\
Kernel RankSVM & \shortstack{.618$^\dag$ \\ {\scriptsize [.579,.674]}} & \shortstack{.659$^\dag$ \\ {\scriptsize [.628,.713]}} \\
LightGBM & \shortstack{.558 \\ {\scriptsize [.509,.616]}} & \shortstack{.619$^\dag$ \\ {\scriptsize [.579,.679]}} \\
\bottomrule
\end{tabular}
\end{center}

\subsection{Per-category NNLS metrics}
\label{app:per-cat-nnls}

Table~\ref{tab:per-category-nnls} supplies the per-category inputs ($\sigma_S/\mu_S$, $\rho_S$, $r$) used in the screening-utility worked example in Section~\ref{sec:theoretical} (Proposition~\ref{main-prep}), together with per-category sample sizes, NNLS pairwise accuracy with $95\%$ BCa CIs, and the rank of the true top-rated product. Expected uplift from a model-driven shortlist scales with $(\sigma_S/\mu_S)\cdot\rho$, so the categories with the largest screening gains are those ranking high on both columns.

\begin{center}
\captionof{table}{Per-category metrics for the NNLS ensemble (BT $+$ Gemini), covering both the per-category evaluation metrics and the inputs used in the worked example (Section~\ref{sec:theoretical}). $n$ = number of products in category; Pairs $= \binom{n}{2}$ within-category ranking pairs. $\sigma_S/\mu_S$ is the coefficient of variation of the panel-mean similarity within the category. $\rho_{\mathrm{S}}$ is Spearman rank correlation; $r$ is Pearson correlation between predicted and panel-mean scores. ``True-best rank'' is the position of the highest-rated product in the model's ranking (1 = top). Pairwise accuracy CIs are 95\% BCa (10{,}000 resamples). Aggregated across these 24 categories, NNLS achieves R@1 $=$ .292, R@2 $=$ .500, R@3 $=$ .542.}
\label{tab:per-category-nnls}
\footnotesize
\setlength{\tabcolsep}{3pt}
\renewcommand{\arraystretch}{1.05}
\begin{tabular}{@{}lrrccccc@{}}
\toprule
Category & $n$ & Pairs & $\sigma_S/\mu_S$ & Pw.\ Acc.\ & $\rho_{\mathrm{S}}$ & $r$ & True-best rank \\
\midrule
Bacon & 10 & 45 & .191 & \shortstack{.533 \\ {\scriptsize [.289,.778]}} & $+$.091 & $+$.235 & 5 \\
Barista Milk & 10 & 45 & .155 & \shortstack{.811 \\ {\scriptsize [.656,.933]}} & $+$.821 & $+$.753 & 1 \\
Bratwurst & 9 & 36 & .169 & \shortstack{.667 \\ {\scriptsize [.375,.875]}} & $+$.450 & $+$.433 & 1 \\
Breaded Chicken Filet & 5 & 10 & .064 & \shortstack{.300 \\ {\scriptsize [.100,.700]}} & $-$.600 & $-$.674 & 4 \\
Breakfast Sausages & 10 & 45 & .125 & \shortstack{.644 \\ {\scriptsize [.333,.900]}} & $+$.382 & $+$.124 & 1 \\
Burgers & 10 & 45 & .190 & \shortstack{.800 \\ {\scriptsize [.600,.933]}} & $+$.758 & $+$.812 & 4 \\
Butter & 10 & 45 & .135 & \shortstack{.711 \\ {\scriptsize [.478,.833]}} & $+$.636 & $+$.602 & 2 \\
Cheddar Cheese & 8 & 28 & .104 & \shortstack{.732 \\ {\scriptsize [.589,.911]}} & $+$.695 & $+$.746 & 2 \\
Chicken Strips & 10 & 45 & .113 & \shortstack{.578 \\ {\scriptsize [.389,.778]}} & $+$.164 & $+$.136 & 4 \\
Cream Cheese & 10 & 45 & .231 & \shortstack{.778 \\ {\scriptsize [.478,.922]}} & $+$.685 & $+$.814 & 7 \\
Creamer & 9 & 36 & .226 & \shortstack{.833 \\ {\scriptsize [.625,.944]}} & $+$.817 & $+$.921 & 4 \\
Deli Ham & 9 & 36 & .160 & \shortstack{.542 \\ {\scriptsize [.250,.819]}} & $+$.117 & $-$.074 & 7 \\
Deli Turkey & 7 & 21 & .183 & \shortstack{.810 \\ {\scriptsize [.571,.952]}} & $+$.750 & $+$.720 & 1 \\
Hot Dogs & 10 & 45 & .191 & \shortstack{.733 \\ {\scriptsize [.567,.889]}} & $+$.697 & $+$.624 & 2 \\
Ice Cream Hard Serve & 10 & 45 & .174 & \shortstack{.533 \\ {\scriptsize [.333,.789]}} & $+$.091 & $+$.432 & 6 \\
Meatballs & 5 & 10 & .256 & \shortstack{.600 \\ {\scriptsize [.100,.900]}} & $+$.300 & $+$.271 & 1 \\
Milk & 18 & 153 & .190 & \shortstack{.830 \\ {\scriptsize [.716,.928]}} & $+$.798 & $+$.794 & 9 \\
Mozzarella & 8 & 28 & .137 & \shortstack{.464 \\ {\scriptsize [.161,.786]}} & $-$.119 & $-$.336 & 8 \\
Nuggets & 10 & 45 & .155 & \shortstack{.333 \\ {\scriptsize [.156,.611]}} & $-$.455 & $-$.275 & 7 \\
Pulled Pork & 7 & 21 & .164 & \shortstack{.571 \\ {\scriptsize [.333,.833]}} & $+$.214 & $+$.197 & 2 \\
Sour Cream & 5 & 10 & .160 & \shortstack{.600 \\ {\scriptsize [.250,.900]}} & $+$.200 & $+$.405 & 2 \\
Steak & 9 & 36 & .214 & \shortstack{.778 \\ {\scriptsize [.542,.903]}} & $+$.733 & $+$.696 & 1 \\
Unbreaded Chicken Breast & 10 & 45 & .137 & \shortstack{.800 \\ {\scriptsize [.600,.911]}} & $+$.782 & $+$.812 & 1 \\
Yogurt & 6 & 15 & .134 & \shortstack{.400 \\ {\scriptsize [.133,.700]}} & $-$.200 & $-$.386 & 3 \\
\bottomrule
\end{tabular}
\end{center}

\subsection{Ingredient Aggregation Ablation}
\label{app:ingredient-agg}

Section~\ref{sec:food-level-features} aggregates compound embeddings to ingredients via a log-concentration-weighted mean, motivated by the Weber--Fechner law. Table~\ref{tab:ingredient-agg-ablation} compares five candidate aggregation rules (linear-concentration, log-concentration, max, arithmetic mean, and top-3-by-concentration) under both Bradley--Terry alone and the BT$+$Gemini NNLS ensemble. Within-block and all-pairs pairwise accuracy are reported with $95\%$ BCa CIs (within-block via cluster bootstrap on blocks). Differences between rules lie within $95\%$ CI overlap on both scopes; we adopt log-concentration on theoretical grounds (Weber--Fechner) rather than empirical advantage.

\begin{center}
\captionof{table}{Compound-to-ingredient aggregation ablation on the food-similarity task (LOOCV, $n=215$ NECTAR products, 935 within-category pairs). For ingredient $I$ with compounds $i \in I$ at concentrations $c_i$ and compound embeddings $z_i$, the per-ingredient embedding $z_I$ follows one of the rules below. \emph{All-pairs} covers all within-category pairs; \emph{within-block} restricts to the BIBD blocks panelists saw (apples-to-apples with humans). 95\% BCa CIs in brackets (10{,}000 resamples; cluster bootstrap on blocks for within-block). \textbf{Bold} = best per (model, scope); $^\dag$ = CI overlaps the leader (no significant difference at 95\%).}
\label{tab:ingredient-agg-ablation}
\footnotesize
\setlength{\tabcolsep}{4pt}
\renewcommand{\arraystretch}{0.95}
\begin{tabular}{@{}llcccc@{}}
\toprule
\multicolumn{2}{c}{} & \multicolumn{2}{c}{BT} & \multicolumn{2}{c}{BT $+$ Gemini NNLS} \\
\cmidrule(lr){3-4} \cmidrule(lr){5-6}
Aggregation & Embedding $z_I$ & All-pairs & Within-block & All-pairs & Within-block \\
\midrule
Uniform mean & $\tfrac{1}{n}\sum_i z_i$ & .592$^\dag$ [.551,.645] & .555$^\dag$ [.487,.616] & .673$^\dag$ [.642,.730] & .656$^\dag$ [.595,.713] \\
Element-wise max & $\max_i z_i$ (per coord.) & .605$^\dag$ [.566,.657] & .572$^\dag$ [.509,.631] & .679$^\dag$ [.649,.740] & \textbf{.661} [.600,.717] \\
Top-3 by conc. & $\tfrac{1}{3}\sum_{i \in \mathrm{top}_3(c)} z_i$ & \textbf{.651} [.621,.702] & .608$^\dag$ [.534,.670] & \textbf{.685} [.657,.737] & .648$^\dag$ [.585,.705] \\
Linear conc. & $\sum_i \tfrac{c_i}{\sum_j c_j}\, z_i$ & .622$^\dag$ [.586,.676] & \textbf{.613} [.554,.670] & .676$^\dag$ [.647,.730] & .651$^\dag$ [.588,.707] \\
Log conc. & $\sum_i \tfrac{\log(1+c_i)}{\sum_j \log(1+c_j)}\, z_i$ & .610$^\dag$ [.571,.662] & .577$^\dag$ [.504,.640] & .683$^\dag$ [.655,.740] & \textbf{.661} [.599,.714] \\
\bottomrule
\end{tabular}
\end{center}

\subsection{Per-class breakdown for molecular prediction}\label{app:molecular-per-class}

Table~\ref{tab:molecular-prediction-per-class} reports per-class precision, recall, F1, and AUROC for FART and our val-best D-MPNN on the $2{,}254$-molecule FART test set. The per-class F1 CIs overlap between the two models on every class: gaps are within $+2.3$\,pp on the four non-umami classes (largest on bitter, where the GNN nominally leads on both precision and recall). On umami, where the macro-F1 gap is concentrated, FART achieves higher precision ($1.000$) at the cost of lower recall ($0.333$) while the GNN achieves balanced precision and recall ($0.667$ each); however, with only $n{=}6$ test molecules a single additional correct prediction shifts F1 by ${\approx}0.17$, and the FART--GNN F1 gap is itself within CI overlap.

\begin{center}
\captionof{table}{Per-class breakdown for Table~\ref{tab:molecular-prediction}: Precision, Recall, F1 and one-vs-rest AUROC of FART and GNN on each FART test class, with 95\% BCa CIs (10{,}000 resamples; point on top, CI in brackets below). Support row gives the number of test molecules per class. Umami CIs span much of $[0, 1]$ because $n=6$. \textbf{Bold} = winner of FART vs GNN per (class, metric); $^\dag$ = CI overlaps the leader (no significant difference at 95\%).}
\label{tab:molecular-prediction-per-class}
\footnotesize
\setlength{\tabcolsep}{3pt}
\renewcommand{\arraystretch}{1.05}
\begin{tabular}{@{}lccccc@{}}
\toprule
 & Sweet & Bitter & Sour & Umami & Undefined \\
\midrule
Support & 1473 & 233 & 238 & 6 & 304 \\
\midrule
\multicolumn{6}{@{}l}{\textit{FART}} \\
\quad Precision & \shortstack{.944$^\dag$ \\ {\scriptsize [.931,.955]}} & \shortstack{.856$^\dag$ \\ {\scriptsize [.801,.901]}} & \shortstack{.897$^\dag$ \\ {\scriptsize [.853,.931]}} & \shortstack{\textbf{1.000} \\ {\scriptsize [.000,1.000]}} & \shortstack{.698$^\dag$ \\ {\scriptsize [.648,.746]}} \\
\quad Recall & \shortstack{.943$^\dag$ \\ {\scriptsize [.930,.954]}} & \shortstack{.712$^\dag$ \\ {\scriptsize [.651,.767]}} & \shortstack{.916$^\dag$ \\ {\scriptsize [.874,.947]}} & \shortstack{.333$^\dag$ \\ {\scriptsize [.000,.800]}} & \shortstack{\textbf{.789} \\ {\scriptsize [.742,.834]}} \\
\quad F1 & \shortstack{.944$^\dag$ \\ {\scriptsize [.935,.952]}} & \shortstack{.778$^\dag$ \\ {\scriptsize [.731,.818]}} & \shortstack{.906$^\dag$ \\ {\scriptsize [.877,.931]}} & \shortstack{.500$^\dag$ \\ {\scriptsize [.000,.909]}} & \shortstack{.741$^\dag$ \\ {\scriptsize [.702,.778]}} \\
\quad AUROC & \shortstack{.979$^\dag$ \\ {\scriptsize [.973,.983]}} & \shortstack{.951$^\dag$ \\ {\scriptsize [.932,.965]}} & \shortstack{.995$^\dag$ \\ {\scriptsize [.992,.997]}} & \shortstack{\textbf{.989} \\ {\scriptsize [.975,.999]}} & \shortstack{.958$^\dag$ \\ {\scriptsize [.946,.967]}} \\
\midrule
\multicolumn{6}{@{}l}{\textit{GNN}} \\
\quad Precision & \shortstack{\textbf{.945} \\ {\scriptsize [.932,.956]}} & \shortstack{\textbf{.869} \\ {\scriptsize [.817,.912]}} & \shortstack{\textbf{.906} \\ {\scriptsize [.864,.937]}} & \shortstack{.667$^\dag$ \\ {\scriptsize [.000,1.000]}} & \shortstack{\textbf{.736} \\ {\scriptsize [.686,.783]}} \\
\quad Recall & \shortstack{\textbf{.953} \\ {\scriptsize [.941,.963]}} & \shortstack{\textbf{.742} \\ {\scriptsize [.683,.796]}} & \shortstack{\textbf{.933} \\ {\scriptsize [.894,.960]}} & \shortstack{\textbf{.667} \\ {\scriptsize [.000,1.000]}} & \shortstack{.770$^\dag$ \\ {\scriptsize [.720,.816]}} \\
\quad F1 & \shortstack{\textbf{.949} \\ {\scriptsize [.940,.957]}} & \shortstack{\textbf{.801} \\ {\scriptsize [.757,.840]}} & \shortstack{\textbf{.919} \\ {\scriptsize [.891,.942]}} & \shortstack{\textbf{.667} \\ {\scriptsize [.000,.889]}} & \shortstack{\textbf{.752} \\ {\scriptsize [.713,.790]}} \\
\quad AUROC & \shortstack{\textbf{.981} \\ {\scriptsize [.975,.985]}} & \shortstack{\textbf{.958} \\ {\scriptsize [.939,.970]}} & \shortstack{\textbf{.996} \\ {\scriptsize [.993,.997]}} & \shortstack{.988$^\dag$ \\ {\scriptsize [.927,1.000]}} & \shortstack{\textbf{.966} \\ {\scriptsize [.958,.973]}} \\
\bottomrule
\end{tabular}
\end{center}

\subsection{Category-balanced evaluation and human comparison}
\label{app:category-balanced}

The main results weight every within-category pair equally, so larger categories
contribute more to the accuracy. Since category sizes range from 5 to 18 products, we report
micro-averaging (as in Table~\ref{tab:results}), macro-averaging across the 24
categories, and DerSimonian--Laird partial pooling toward a precision-weighted grand
mean (Table~\ref{tab:micro-macro}).

The best model ranks first under all three aggregation methods. Macro-averaging lowers every model by .000 to .055 and largely preserves the
ordering (Spearman .881). The partial-pooling estimate (.682) is close to the micro
average because ${\approx}91\%$ shrinkage ($\tau^{2} = .00208$) pulls toward the
pooled mean, where larger categories contribute more. Excluding the smallest categories does
not change the best model (Table~\ref{tab:exclusion-sweep}).
Table~\ref{tab:human-comparison} compares each model to the median panelist on
matched within-block pairs.

\begin{table}[t]
\centering
\footnotesize
\setlength{\tabcolsep}{6pt}
\renewcommand{\arraystretch}{0.95}
\caption{\textbf{Category-balanced evaluation preserves the ranking and the best
model.} Pairwise accuracy on the 935 within-category pairs under three aggregation
schemes. Micro is the paper's reported metric (Table~\ref{tab:results}); macro weights
each of the 24 categories equally; EB is DerSimonian--Laird partial pooling toward the
precision-weighted random-effects grand mean, with each category's sampling variance
taken from its product count. $\Delta$ $=$ macro $-$ micro. Point estimates only:
intervals for the macro and EB columns come from a within-category percentile bootstrap
rather than the BCa procedure used elsewhere in the paper, so they are not comparable to
the intervals in Table~\ref{tab:results} and are omitted here. Spearman correlation
between the macro and micro orderings is .881.}
\label{tab:micro-macro}
\begin{tabular}{@{}lcccc@{}}
\toprule
Model & Micro & Macro & $\Delta$ & Partial pooling (EB) \\
\midrule
\multicolumn{5}{@{}l}{\emph{Ensemble (BT $+$ Gemini)}} \\
\quad NNLS & \textbf{.683} & \textbf{.641} & $-$.042 & \textbf{.682} \\
\quad Rank average & .671 & .629 & $-$.041 & .668 \\
\quad Mean & .623 & .595 & $-$.028 & .613 \\
\multicolumn{5}{@{}l}{\emph{Zero-shot LLMs}} \\
\quad Gemini 3.1 Pro & .654 & .633 & $-$.021 & .666 \\
\quad Qwen 3.5 397B-A17B & .630 & .611 & $-$.019 & .639 \\
\multicolumn{5}{@{}l}{\emph{Supervised}} \\
\quad Kernel RankSVM & .618 & .587 & $-$.031 & .612 \\
\quad Hierarchical BT & .613 & .558 & $-$.055 & .595 \\
\quad Bradley--Terry & .610 & .567 & $-$.043 & .567 \\
\quad Ridge & .579 & .537 & $-$.042 & .540 \\
\quad LightGBM & .558 & .532 & $-$.026 & .555 \\
\multicolumn{5}{@{}l}{\emph{Unsupervised}} \\
\quad MMRF (cosine) & .574 & .573 & $-$.002 & .589 \\
\quad MMRF (L2) & .571 & .570 & $-$.000 & .600 \\
\bottomrule
\end{tabular}
\end{table}

\begin{table}[t]
\centering
\footnotesize
\setlength{\tabcolsep}{6pt}
\renewcommand{\arraystretch}{0.95}
\caption{\textbf{Dropping small categories does not change the best model.} The best
model (NNLS over Bradley--Terry and Gemini 3.1 Pro) as categories below each
product-count threshold are excluded. Rank is among the twelve models of
Table~\ref{tab:micro-macro}, under micro and macro averaging respectively. The
smallest category has 5 products, so the first row is the full benchmark.}
\label{tab:exclusion-sweep}
\begin{tabular}{@{}ccccc@{}}
\toprule
Min.\ products & Categories kept & Micro & Macro & Rank (micro / macro) \\
\midrule
5 & 24 & .683 & .641 & 1 / 1 \\
6 & 21 & .689 & .661 & 1 / 1 \\
7 & 20 & .694 & .674 & 1 / 1 \\
8 & 18 & .694 & .672 & 1 / 1 \\
9 & 16 & .701 & .682 & 1 / 1 \\
10 & 12 & .700 & .674 & 1 / 1 \\
\bottomrule
\end{tabular}
\end{table}

\begin{table}[t]
\centering
\footnotesize
\setlength{\tabcolsep}{6pt}
\renewcommand{\arraystretch}{0.95}
\caption{\textbf{The model against individual panelists, on matched pairs.}
Agreement quantities for the panel and for the best model. The panelist rows are
computed on the within-block pairs each panelist actually rated; the model is scored on
those same pairs in the first model row, and on all 935 within-category pairs in the
second. The category-balanced row is the partial-pooled figure from
Table~\ref{tab:micro-macro}. Rows are not mutually comparable: the panelist median is
a per-panelist quantity on within-block pairs, not a category-macro one.}
\label{tab:human-comparison}
\begin{tabular}{@{}lc@{}}
\toprule
Quantity & Value \\
\midrule
\multicolumn{2}{@{}l}{\emph{Panel}} \\
\quad Panelist vs.\ another single panelist, pairwise ordering & .547 \\
\quad Median panelist vs.\ panel mean, within block & .650 (IQR .500--.750, $n = 627$) \\
\quad Krippendorff's $\alpha$ (ordinal, per block) & .077 (mean) \\
\quad Kendall's $W$ (rank concordance, per block) & .108 (median) \\
\quad Split-half reliability of the panel mean & .825 \\
\midrule
\multicolumn{2}{@{}l}{\emph{Best model}} \\
\quad Within-block pairs & .661 \\
\quad All within-category pairs & .683 \\
\quad Category-balanced, partial-pooled & .682 \\
\bottomrule
\end{tabular}
\end{table}

\subsection{Recognition and contamination audit}
\label{app:recognition-audit}

We audit whether the best LLM configuration (ingredients and image) recognizes
specific products rather than reasoning from composition, and whether
training-data contamination could explain its advantage.

\mypar{What surfaces in the reasoning} We scan the saved chain-of-thought for four
signatures (Table~\ref{tab:recognition-signatures}). The model identifies a specific
product in 12.5\% of pairs and references knowledge about a product's reputation
in 2.4\%; neither signature is associated with higher accuracy. These
identifications are inference, not input leakage, since the cleaned ingredient
text contains no brand tokens and the images show plated food.

\begin{table}[t]
\centering
\footnotesize
\setlength{\tabcolsep}{6pt}
\renewcommand{\arraystretch}{0.95}
\caption{\textbf{Recognition rarely surfaces, and it does not track accuracy.}
Prevalence of each signature in the saved reasoning for the ingredients-plus-image
configuration, and the accuracy difference between the pairs where it appears and the
rest. Scored on the 931 pairs for which a graded judgment and a Bradley--Terry
comparison both exist: of the 935 within-category pairs, three have tied panel means and
are not gradable, and one more is unmatched. Fisher's exact test, two-sided; no
signature reaches significance. Signatures are not mutually exclusive.}
\label{tab:recognition-signatures}
\begin{tabular}{@{}lccc@{}}
\toprule
Signature & Prevalence & $\Delta$ accuracy (flagged $-$ clean) & Fisher $p$ \\
\midrule
Product identification & 12.5\% & $+$.042 & .403 \\
Brand mention (mostly analogy) & 10.2\% & $+$.013 & .820 \\
External-performance recall & 2.4\% & $-$.026 & .821 \\
Any signature & 14.2\% & $+$.032 & .489 \\
\bottomrule
\end{tabular}
\end{table}

\mypar{Margin decomposition} If recognition helped, the Gemini$-$BT margin should be
larger on recognized pairs, since Bradley--Terry sees only features and cannot benefit
from recognition (Table~\ref{tab:recognition-margin}). The margin is if anything
smaller on recognized pairs, and the interaction is not significant.
Recognized pairs appear to be intrinsically easier; even Bradley--Terry
scores .672 on them versus .605 on the rest.

\begin{table}[t]
\centering
\footnotesize
\setlength{\tabcolsep}{6pt}
\renewcommand{\arraystretch}{0.95}
\caption{\textbf{The LLM's margin is not concentrated where it recognized the
product.} Per-pair margin of Gemini 3.1 Pro over Bradley--Terry, split by whether the
reasoning identified the product. Bradley--Terry sees only features, so it holds pair
difficulty constant. Over these 931 pairs Gemini scores .662 and
Bradley--Terry .613, a margin of $+$.048; both differ slightly
from Table~\ref{tab:results}, which scores the same models over all 935 pairs
including the three ties. Bradley--Terry also improves on the recognized subset
(.605 to .672), which is what identifies those pairs as easier rather
than leaked. The interaction interval is a within-category bootstrap.}
\label{tab:recognition-margin}
\begin{tabular}{@{}lcc@{}}
\toprule
Split & $n$ & Gemini $-$ Bradley--Terry \\
\midrule
Recognized & 116 & $+$.026 \\
Not recognized & 815 & $+$.052 \\
\midrule
Interaction & & $-$.026, 95\% CI [$-$.159,$+$.108] (n.s.) \\
\bottomrule
\end{tabular}
\end{table}

\mypar{An interventional control} We re-ran the benchmark with only the system prompt
changed, instructing the model to reason from composition and appearance only; the user
prompt is byte-identical (Table~\ref{tab:brand-blind}). Brand mentions drop from
10.8\% to 1.7\%; accuracy changes by $-$.005 (.639 to .634, McNemar $p = .71$). Both
arms must be run in the same session because the unconstrained prompt flips
${\sim}29\%$ of predictions across runs at temperature 0, and comparing against a
previously cached baseline would show a spurious $-$.028 ($p = .03$).

\begin{table}[t]
\centering
\footnotesize
\setlength{\tabcolsep}{5pt}
\renewcommand{\arraystretch}{0.95}
\caption{\textbf{Forbidding brand reasoning barely moves accuracy.} Original versus
brand-blind system prompt, both re-run fresh in the same session so that model drift
cancels; the user prompt is byte-identical between arms, and 932 pairs are scored
in both. Flip rate is the fraction of pairs whose predicted winner changes. The
brand-mention rate is measured on the same reasoning scan as
Table~\ref{tab:recognition-signatures}.}
\label{tab:brand-blind}
\begin{tabular}{@{}lcccccc@{}}
\toprule
Arm & Original & Brand-blind & $\Delta$ acc.\ & Flip rate & McNemar $p$ & Brand mentions \\
\midrule
Gemini 3.1 Pro & .639 & .634 & $-$.005 & 12.6\% & .712 & 10.8\% $\to$ 1.7\% \\
\bottomrule
\end{tabular}
\end{table}

\mypar{Training-data contamination} The dairy TASTY winners were released on
2026-03-18, after both model releases (Gemini 3.1 Pro 2026-02-19, Qwen
3.5-397B-A17B 2026-02-16), and Gemini's documented January 2025 training cutoff
also predates the March 2025 meat winners. The per-set accuracies are not evidence
of memorization, since the memorization-free Bradley--Terry baseline shows the same
dairy-over-meat gap.

\mypar{Caveats} The three recognition tests rely on what surfaces in the
reasoning, and chain-of-thought can be unfaithful, so they bound a latent
contribution rather than excluding it. In particular, they cannot exclude
the possibility that the LLM
draws on a product's pre-existing reputation when that reputation correlates with true
taste performance. Post-release NECTAR collections will provide a prospective
hold-out.

\section{Architecture with Worked Example}\label{app:architecture-diagram}

Figure~\ref{fig:architecture} traces a single product, a plant-based burger
in the Burgers category (``Product~A''), through the full TasteBench ensemble
pipeline. The same pipeline is applied independently to every product, and we
show one instance to make the data flow concrete. All four modalities
(ingredient text, nutrition, product image, and molecular composition) are
encoded into per-modality vectors, which are standardized and, for the
high-dimensional text, image, and compound blocks, reduced by per-modality PCA
to 95\% explained variance (fit within each LOOCV fold). The low-dimensional
nutrition and category-subset blocks are not PCA-reduced. The standardized,
concatenated SNCTI vector is the input to the supervised component (Bradley--Terry,
the supervised model carried into the final ensemble) fit by logistic
regression on within-category feature-pair differences and scored at inference
as a product's expected number of within-category tournament wins,
$\sum_{k} \sigma\!\big(w\cdot(x - x_k)\big)$ over the same-category opponents
$x_k$. The non-molecular modalities are additionally read in raw form by the
zero-shot LLM scorer (Gemini~3.1~Pro), which does not consume the encoded or
fused features; the LLM score is a product's win rate over
its within-category pairwise comparisons. The two scores, $s_{\mathrm{BT}}$ and
$s_{\mathrm{LLM}}$, are combined by non-negative least squares
($w_1, w_2 \geq 0$, fit by nested leave-one-out cross-validation) to produce the
predicted within-category ranking. The figure's ``Gemini inputs'' grouping marks
the modalities the LLM can read directly from the prompt template
(Appendix~\ref{app:llm-prompts}). The deployed ensemble uses the ingredients-plus-image
configuration, with the ingredients-plus-nutrition-plus-image variant evaluated
as an ablation (Table~\ref{app:per-cat-comparison}). Ingredient and nutrition values are
abbreviated for readability, and the nine opponents shown ($X_1,\dots,X_9$) are
specific to the ten-product Burgers category, the sizes of other categories vary.

\begin{figure*}[tp]
\centering
\includegraphics[width=\textwidth]{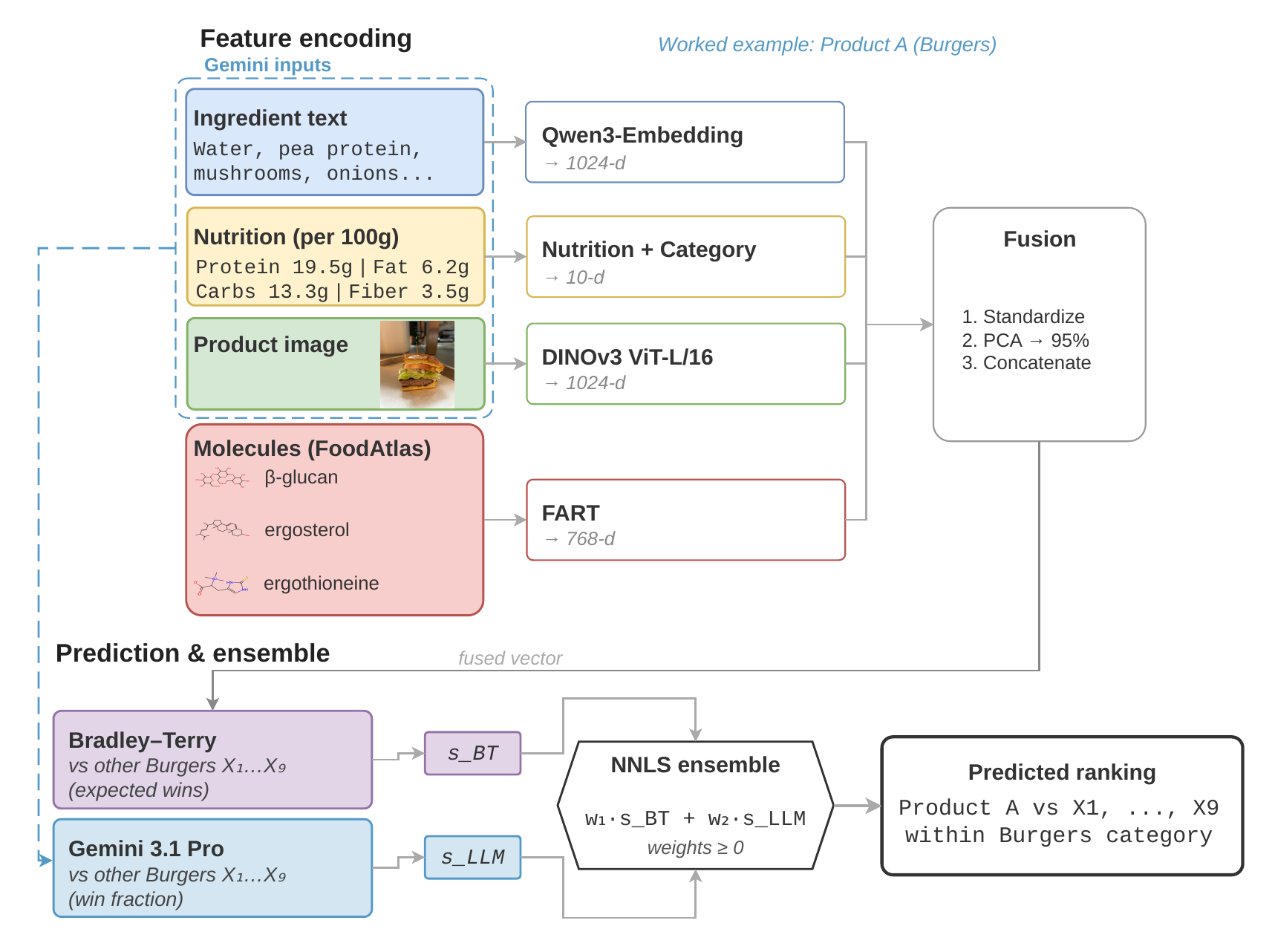}
\caption{Worked example of the ensemble pipeline for one plant-based burger (Product~A) in the Burgers category ($n=10$). All four modalities feed the supervised Bradley--Terry scorer through encoding and fusion; the non-molecular modalities are also read in raw form by the zero-shot LLM scorer (Gemini~3.1~Pro). The two scores are combined via non-negative least squares to produce the predicted within-category ranking.}
\label{fig:architecture}
\end{figure*}



\end{document}